\documentclass{article} % For LaTeX2e
\usepackage{iclr2024_conference,times}

\usepackage{amsmath,amsfonts,bm}

\def\eqref#1{equation~\ref{#1}}
\def\1{\bm{1}}

\DeclareMathAlphabet{\mathsfit}{\encodingdefault}{\sfdefault}{m}{sl}
\SetMathAlphabet{\mathsfit}{bold}{\encodingdefault}{\sfdefault}{bx}{n}

\usepackage{url}
\usepackage[utf8]{inputenc} % allow utf-8 input
\usepackage[T1]{fontenc}    % use 8-bit T1 fonts
\usepackage{url}            % simple URL typesetting
\usepackage{booktabs}       % professional-quality tables
\usepackage{hyperref}       % hyperlinks
\usepackage{amsfonts}       % blackboard math symbols
\usepackage{nicefrac}       % compact symbols for 1/2, etc.
\usepackage{microtype}      % microtypography
\usepackage{xcolor}         % colors

\usepackage{amsmath}
\usepackage{amssymb}
\usepackage{amsthm}
\usepackage[titlenumbered, linesnumbered, ruled]{algorithm2e}

\usepackage{multirow}
\usepackage{paralist}
\usepackage{graphicx}
\usepackage{makecell}

\theoremstyle{plain}
\newtheorem{theorem}{Theorem}
\newtheorem{proposition}[theorem]{Proposition}
\newtheorem{lemma}[theorem]{Lemma}

\newtheorem{definition}[theorem]{Definition}
\newtheorem{assumption}[theorem]{Assumption}
\newtheorem{example}[theorem]{Example}
\newtheorem{property}[theorem]{Property}
\newtheorem{remark}[theorem]{Remark}

\newcommand{\entity}{\mathcal{E}}
\newcommand{\relation}{\mathcal{R}}
\newcommand{\efo}{{\textsc{EFO}_1}}

\newcommand{\tf}{{\textsc{TF}}}

\DeclareMathOperator{\diag}{{\rm diag}}

\title{Rethinking Complex Queries on Knowledge Graphs with Neural Link Predictors}

\author{Hang Yin\\
   Department of Mathematical Sciences\\
   Tsinghua University \\
   \texttt{h-yin20@mails.tsinghua.edu.cn} \\
   \And
    Zihao Wang \\
  Department of CSE\\
  HKUST\\
  \texttt{zwanggc@cse.ust.hk} \\
  \And
  Yangqiu Song\\
  Department of CSE \\
  HKUST\\
  \texttt{yqsong@cse.ust.hk}
}

\iclrfinalcopy % Uncomment for camera-ready version, but NOT for submission.
\begin{document}

\maketitle

\begin{abstract}
Reasoning on knowledge graphs is a challenging task because it utilizes observed information to predict the missing one. Particularly, answering complex queries based on first-order logic is one of the crucial tasks to verify learning to reason abilities for generalization and composition.
Recently, the prevailing method is query embedding which learns the embedding of a set of entities and treats logic operations as set operations and has shown great empirical success. Though there has been much research following the same formulation, many of its claims lack a formal and systematic inspection. In this paper, we rethink this formulation and justify many of the previous claims by characterizing the scope of queries investigated previously and precisely identifying the gap between its formulation and its goal, as well as providing complexity analysis for the currently investigated queries. Moreover, we develop a new dataset containing ten new types of queries with features that have never been considered and therefore can provide a thorough investigation of complex queries. Finally, we propose a new neural-symbolic method, Fuzzy Inference with Truth value (FIT), where we equip the neural link predictors with fuzzy logic theory to support end-to-end learning using complex queries with provable reasoning capability. Empirical results show that our method outperforms previous methods significantly in the new dataset and also surpasses previous methods in the existing dataset at the same time.
\end{abstract}

\section{Introduction}

Knowledge graph (KG) is a mighty knowledge base that encodes relational knowledge into a graph representation. However, due to the fact that modern knowledge graphs are often auto-generated~\citep{toutanova_observed_2015} or constructed by crowd-sourcing~\citep{vrandecic_wikidata_2014}, they are considered noisy and incomplete, which is also known as the Open World Assumption (OWA)~\citep{libkin_open_2009}. 
Complex query answering (CQA) on knowledge graphs is a practical task that can support many applications~\citep{ren_neural_2023,wang_logical_2022}. The CQA task requires answering the existential first order logic formula, involving logical operators, conjunction ($\land$), disjunction ($\lor$), and negation ($\lnot$), as well as the existential quantifier  ($\exists$).
In particular, as CQA is based on KGs with OWA, it should perform reasoning, which utilizes available knowledge to predict the missing one where traditional traversal methods are doomed to fail~\citep{ren_query2box_2020}.

To tackle this challenge, the query embedding method has been proposed~\citep{hamilton_embedding_2018}, which aims to represent a set of entities by a low dimensional embedding. In addition to that, the logical formula is represented in an operator tree form~\citep{ren_query2box_2020, wang_benchmarking_2021}, in which the logic operations are replaced with corresponding set operations. Especially, the existential quantifier induces a new set operation, set projection, which corresponds to the logic skolemization~\citep{luus_logic_2021}. 
Though there has been a line of research~\citep{choudhary_probabilistic_2021,bai_query2particles_2022,yang_gammae_2022,wang_wasserstein-fisher-rao_2023,hu_fedcqa_2024} following the same approach, this kind of methodology still lacks detailed inspection of its logic soundness or model expressiveness. Moreover, despite the operator tree form pushing a strict constraint on the solvable formulas, the real scope of the solvable logic formulas has never been estimated, and nowadays datasets are therefore highly restricted. In general, not much theoretical progress has been made to inspect on nowadays CQA models.

In this paper, we first review the inherent drawbacks of the existing prevailing operator tree form representation of query embedding approaches and clearly characterize the query family that has been investigated as Tree-Form ($\tf$) queries. Then we extend our scope of research to the whole family of {\it Existential First Order queries with
a single free variable} ($\efo$). Particularly, we represent queries as general multigraphs.
Then we develop a new dataset containing ten new formulas which can not be represented in prevailing frameworks. Along with that, we propose a simple yet empirically effective algorithm which combines neural link predictors with strict fuzzy logic definition, and thus has strong theoretical guarantees. The algorithm is able to systematically infer $\efo$ queries of arbitrary complexity given any neural link predictor. Finally, we show that our algorithm is able to outperform existing methods in both our newly developed dataset and existing dataset. Our code and data can be found at~\url{https://github.com/HKUST-KnowComp/FIT}.

\section{Preliminary}

\subsection{Knowledge graphs}
Given a set of entities $\entity$ and a set of relations $\relation$, a knowledge graph $\mathcal{G}$ encapsulates real-world knowledge as a collection of factual triples $\mathcal{G} = \{(a_i,r_i,b_i)\}$, in each triple,  $a_i,b_i \in \entity$ is the head entity and tail entity correspondingly, and $r_i\in \relation$ is the relation among them. Based on OWA, the \textbf{observed} knowledge graph $\mathcal{G}_{o}$ is only part of the \textbf{complete} knowledge graph $\mathcal{G}$, meaning $\mathcal{G}_{o} \subsetneq \mathcal{G}$.

\subsection{$\efo$ queries and answers}
Existing discussions emphasize logical queries without universal quantifiers~\citep{ren_beta_2020}.
Such queries are formally defined as Existential First Order queries with a single free variable ($\efo$) under the strict first-order logic theory. We provide a minimum set of definitions to characterize the $\efo$ queries on knowledge graphs following standard logic theory~\citep{marker_model_2006}.

\begin{definition}[Terms]
    A term is either a variable $x$ or an entity $a\in \entity$.
\end{definition}

\begin{definition}[Atomic Formula]\label{def:atomic-formula}
    An atomic formula is of form $\phi= r(h, t)$,  $r\in \relation$ , $h$ and $t$ are terms. 
    %If $h,t\in \entity$, given the knowledge graph,  $r(h,t)$ is True if and only if the tuple $(h, r, t)\in \mathcal{G}$.
\end{definition}

\begin{definition}[Existential First Order Formula]\label{def:formula}
The set of the existential formulas is the smallest set $\Phi$ that satisfies the following property:
\begin{compactitem}
    \item[(i)] For atomic formula $r(a,b)$, itself and its negation $r(a,b), \lnot r(a,b)\in \Phi$
    \item[(ii)] If $\phi, \psi\in \Phi$, then $(\phi\land \psi),(\phi \lor \psi) \in \Phi$
    %\item[(iii)] If $\phi \in \Phi$, then $\lnot \phi \in \Phi$; 
    \item[(iii)] If $\phi \in \Phi$ and $x_i$ is any variable, then $\exists x_i \phi \in \Phi$.
\end{compactitem}
\end{definition}

We say a variable $y$ is bounded if there is an associated quantifier, otherwise, it is free. We use $\phi(y)$ to indicate the formula $\phi$ contains a free variable $y$. Then we finish the definition of the $\efo$ formula.

\begin{remark}[Query and sentence]
    When a formula contains at least one free variable, it is also called a query. Otherwise, the formula can be called a sentence. 
\end{remark}
\begin{definition}[Substitution]
    For a $\efo$ formula $\phi(y)$, for any entity $a\in \entity$, we write $\phi(a/y)$ or simply $\phi(a)$, for the result of simultaneously replacing all free occurrence of $y$ in $\phi$ by $a$.
\end{definition}

Then, we are ready to define $\efo$ queries and the answer set.
\begin{definition}[The Answer Set of $\efo$ Query]\label{def:efo1-query}
    The answer set of an $\efo$ query is defined by
    \begin{align}
        \mathcal{A}[\phi(y)] = \{a\in \entity|\text{ }\phi(a) \text{ is True}\}
    \end{align}
\end{definition}

All $\efo$ formulas can be converted to Disjunctive Normal Form (DNF), specifically:
\begin{definition}[Disjunctive Normal Form]
    The disjunctive normal form $\phi_{\rm DNF}$ of an $\efo$ formula is
    \begin{align}\label{eq:efo-in-dnf}
        \phi_{\rm DNF}(y)=  \gamma_1(y) \lor \dots \lor \gamma_m(y),
    \end{align}
    where $\gamma_j(y) = \exists x_1,\cdots, x_k. \alpha_{j1} \land \dots \land \alpha_{j n_j}$, $j = 1,...,m$, where $x_i$, $i=1,...,k$ are existential variables, $y$ is the single free variable, $\alpha_{j\cdot}$  are atomic formulas or the negation of atomic formulas.
\end{definition}
There are many ways to represent a $\efo$ formula, however, the DNF has gained dominant priority in lines of previous research~\citep{ren_beta_2020,wang_benchmarking_2021,zhang_cone_2021}. Converting the original query $\phi(y)$ into a DNF query $\phi_{\rm DNF}(y)$ decomposes the answer set $\mathcal{A}[\phi_{\rm DNF}(y)]$ into the union of the answer sets of $\mathcal{A}[\gamma_j(y)]$, which has become a common practice in recent research~\citep{wang_logical_2023}.
We consider $\efo$ queries in DNF throughout this paper.

\section{The limitation of previous formulation}\label{sec:tf-to-efo}

In this section, we aim to answer the question: Can general $\efo$ queries be precisely represented by previous formulations in syntax? In recent years, a huge number of research~\citep{ren_beta_2020,wang_benchmarking_2021,zhang_cone_2021, ren_fact_2023} have all made such a strong claim. However, derived from the design of the previous method, we are able to characterize the query family that has been investigated as Tree-Form ($\tf$) queries and provide it with a formal definition. Our discussion shows that the $\tf$ query family is not even a subset of $\efo$. Therefore, our discussion justifies the limitation of previous works by showing their unrigorous formulation.

\subsection{The syntactical closure of previous formulation: Tree-form queries}
The query types in the existing datasets~\citep{ren_beta_2020,wang_benchmarking_2021}, though targeted to $\efo$ query family, are selected with \textit{bias} when it comes to the empirical evaluation. 
The reason is that the dominating way of addressing logical queries is to simulate logical reasoning as the execution of set operators on an \textbf{operator tree}~\citep{wang_benchmarking_2021,xu_neural-symbolic_2022}, where each node represents a set of entities corresponding to the answer set of a sub-query. The logical connectives are transformed into operator nodes for set projections, intersection, union, and complement~\citep{wang_benchmarking_2021}. Particularly, the set projections are derived from the Skolemization of predicates~\citep{luus_logic_2021}.

We characterize the expressiveness of operator tree method by providing the formal definition of the $\tf$ query, instead of confusing them with $\efo$ queries as the previous practices~\citep{wang_benchmarking_2021}.

\begin{definition}[Tree-Form Query]\label{def:TF query}
The set of the Tree-Form queries is the smallest set $\Phi_\tf$ such that: 
\begin{compactitem}
\item[(i)] If $\phi(y) = r(a,y)$, where $a\in \entity$, then $\phi(y) \in \Phi_\tf$; 
\item[(ii)] If $\phi(y)\in \Phi_\tf, \lnot \phi(y) \in \Phi_\tf$; 
\item[(iii)] If $\phi(y), \psi(y)\in \Phi_\tf$, then $(\phi\land \psi)(y) \in \Phi_\tf$ and $(\phi \lor \psi)(y)\in \Phi_\tf$;
\item[(iv)] If $\phi(y) \in \Phi_\tf$ and $y^{\prime}$ is any variable, then $\psi(y^{\prime})= \exists  y. r(y,y^{\prime})\land \phi(y) \in \Phi_\tf$.    
\end{compactitem}
\end{definition}

Figure~\ref{fig:pni} is an example of a Tree-Form query that appeared in the widely used dataset~\citep{ren_beta_2020}. The formal derivation of the definition of $\tf$ can be found in the Appendix~\ref{app:tree form derivation}.

\begin{figure}   
    \vspace{-3em}
    \begin{minipage}{0.48\textwidth}
     \centering
     \includegraphics[width=\linewidth]{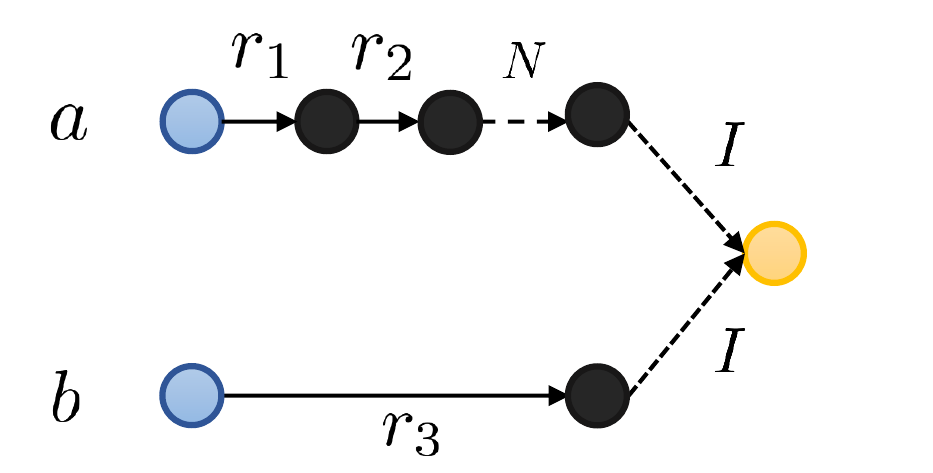}
     \caption{Representation of the tree form query ``pni''. We note that this kind of representation requires explicit set operators in the graph, the corresponding lines are dotted.}\label{fig:pni}
   \end{minipage}\hfill
   \begin{minipage}{0.48\textwidth}
     \centering
     \includegraphics[width=\linewidth]{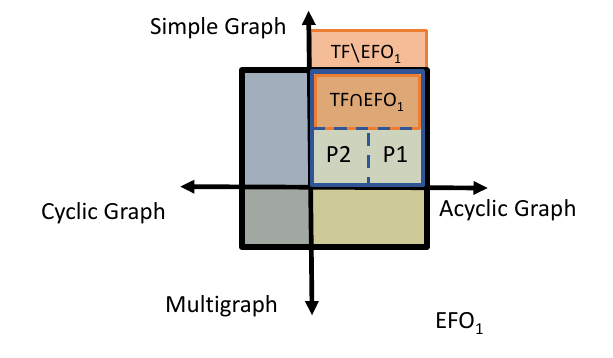}
     \caption{A diagram for the differences between the Tree-Form queries (orange blocks) and the $\efo$ queries (black box). $\efo$ queries are categorized by their query graphs. Some Tree-Form queries are not of $\efo$.}\label{fig:query-diagram}
   \end{minipage}
    \vspace{-1em}
\end{figure}

However, this biased selection of query types \textbf{deviates} from the original goal of complex query answering~\citep{ren_beta_2020}. It imposes strong assumptions for the formula~\citep{wang_benchmarking_2021}, so a large part of first-order queries, even only with existential quantifiers, is ignored in both existing datasets and the solving strategies.
Moreover, it is questionable whether existing query types indeed skip the universal quantifier. As we have pointed out in the following, the recursive definition may seem to neglect the universal quantifier but fails when the formula becomes more complex.

% However, this kind of method fails when the negation is introduced by~\citet{ren_beta_2020}, but the problem is unrecognized to the best of our knowledge. We illustrate this problem with the following example.
\subsection{$\tf$ queries deviate from the goal of $\efo$ queries}
\begin{proposition}\label{prop:bad-TF}
    In DNF\footnote{More generally, this conclusion is true in every prenex normal form, which is a pre-requisite for query embedding method~\citep{wang_benchmarking_2021}.}, the universal quantifier may exist in $\tf$ queries, thus the family of $\tf$ query is not a subset of $\efo$.
\end{proposition}
\noindent{Proof:} We derive this proposition based on the Definition~\ref{def:TF query}:
\begin{align*}
    r_1(a,x) &\in \Phi_\tf \\ 
    \exists x. r_1(a,x)\land r_2(x,y)& \in \Phi_\tf \\
    \lnot \exists x. r_1(a,x)\land r_2(x,y)& \in \Phi_\tf \\
    \forall x. \neg r_1(a, x) \lor \neg r_2(x, y) & \in \Phi_\tf 
\end{align*} 
%Considering a formula $\phi(y) =\exists x \psi(x,y)$, where $\psi(x,y)$ is also an existential formula, let its answer set $A[\phi(y)]=B$. Notice that tree-from queries utilize set complement to tackle negation, we consider $\overline{B}$, namely the complement of set $B$,  $b\in \overline{B}$ means $b \notin B$, which says exactly that $\phi(b/y)$ is false, so we know that $\lnot \exists x \psi(x, b)=\forall x  \lnot\psi(x,b)$  is true.
%That is $\forall x  \lnot \psi(x,y)$ have the answer $b$, thus $A[\forall x  \lnot \psi(x,y)]= \overline{B}$.
%This problem occurs due to taking the negation on the whole formula changes the inner existential quantifiers to universal quantifiers.

We note that in Definition~\ref{def:formula}, the negation is defined only on the atomic formulas, while in Definition~\ref{def:TF query}, the negation is defined on the whole formula, which leads to the occurrence of the universal quantifier.

\begin{example}\label{example:pni}
Based on the above discussion, the original ``pni'' query proposed in ~\citet{ren_beta_2020} and shown in Figure~\ref{fig:pni}, should have this kind of formulation in DNF ($a,b\in \entity$): 
\begin{equation}
    \forall x. (r_3(b, y)\land \neg r_2(x, y)) \lor (r_3(b, y)\land \neg r_1(a, x))
\end{equation}
\end{example}

By proposition~\ref{prop:bad-TF} and Example~\ref{example:pni}, we clearly show that both the methodology and dataset in previous research~\citep{ren_beta_2020} deviate from its original ambitious goal to answer $\efo$ query.

\section{Gaps between $\tf$ and $\efo$ query}
In this section, we aim to further identify the gap between the $\efo$ query family and $\tf$ family. To this end, we first introduce the representation method for general $\efo$ query.
\subsection{A graph-theoretical description of $\efo$ queries}\label{sec:Unsolvable}

We discuss $\efo$ queries by considering each conjunctive query $\gamma_j$ as a query graph. 
Graphs representation is already discussed in answering Tree-Form queries without negation~\citep{daza_message_2020,liu_mask_2022} and Tree-Form queries~\citep{wang_logical_2023}.

\begin{definition}[Query Graph]
    Let $\gamma$ be a conjunctive formula in ~\eqref{eq:efo-in-dnf}, its query graph $G(\gamma) = \{(h,r,t,\{\text{T/F}\})\}$, each quadruple corresponds to an atomic formula or its negation, representing an edge with two endpoints $h,t$, and two attributes $r$, \text{T/F}, the relation and whether it is positive.
\end{definition}

% where $V$ is the set for all the terms in $\gamma$ and $E$ is the set for all atomic formulas in $\gamma$, with or without negation. Specifically, for an atomic formula $r(h,t)$ where h,t are two terms, will correspond to an edge $r$ from node $h$ to node $t$, if it there is an negated atomic formula $\lnot r(h,t)$, the edge is labeled as negative edge.   
We note nodes of different types represent the node to be a constant entity, existential variable, or free variable. Logical conjunction is naturally presented in the query graph because the order of existentially quantified variables is exchangeable, indicating the main strength of the query graph representation. 
Then, all $\efo$ queries in DNF can be represented by a set of query graphs $\{G(\gamma_j)\}_{j=1}^{m}$. We show an example of query graph in Figure~\ref{fig:query graph}.

The concept of query graph is also similar to the Gaifman graph~\citep{vardi_constraint_2000}, which is well investigated in constraint programming. The key differences are (1) query graphs emphasize more node types for each term and (2) query graphs also encode logical negation.

\subsection{Syntatical gap between $\tf$ and $\efo$ queries}\label{sec: gap between tree-form and efo}

Figure~\ref{fig:query-diagram} briefly categorizes $\efo$ queries by whether their query graphs are acyclic or simple. We note previous research shows that Tree-Form queries are assumed to be both acyclic and simple~\citep{hamilton_embedding_2018}. Then, the key to understanding the relationship between $\tf$ queries and $\efo$ queries is to determine the gap between $\tf$ queries without a universal quantifier (the bottom half orange box) and $\efo$ queries with an acyclic simple query graph (the top-right of the dark box). We characterize this gap by the following two properties.

%Moreover, the tree form requires many assumptions of the existential formulas which are still left unexplored, as illustrated by~\citet{hamilton_embedding_2018,liu_mask_2022}, it requires a directed acyclic graph(DAG) structure to allow for searching the answers in a bottom-up manner. We illustrate those additional assumptions in the following.

% Based on the DNF, we are allowed only considering conjunctive formula, which has the form 

% To facilitate our discussion, we propose to represent this kind of query by the idea of a ``query graph'', which has been proposed~\citep{daza_message_2020} and extended ~\citep{wang_logical_2023} previously. This method encodes the query without any explicit logic operator and can generalize to existential queries of arbitrary complexity.
% We provide an example of the query graph shown in Figure~\ref{fig:query graph}.

%$\exists x_1\exists x_2 \exists x_3, \textbf{Award}(\textrm{OutstandingPaper}, x_1)\land \textbf{CorrespondingAuthor}(x_1, y)\land \textbf{Cite}(x_1, x_2) \land \textbf{PublishIn}(x_2,x_3) \land \textbf{Author}(x_2, y)\land \lnot \textbf{FirstAuthor}(x_2, y)$

\begin{figure*}[t]
    \vspace{-3em}
    \centering
    \includegraphics[width=\linewidth]{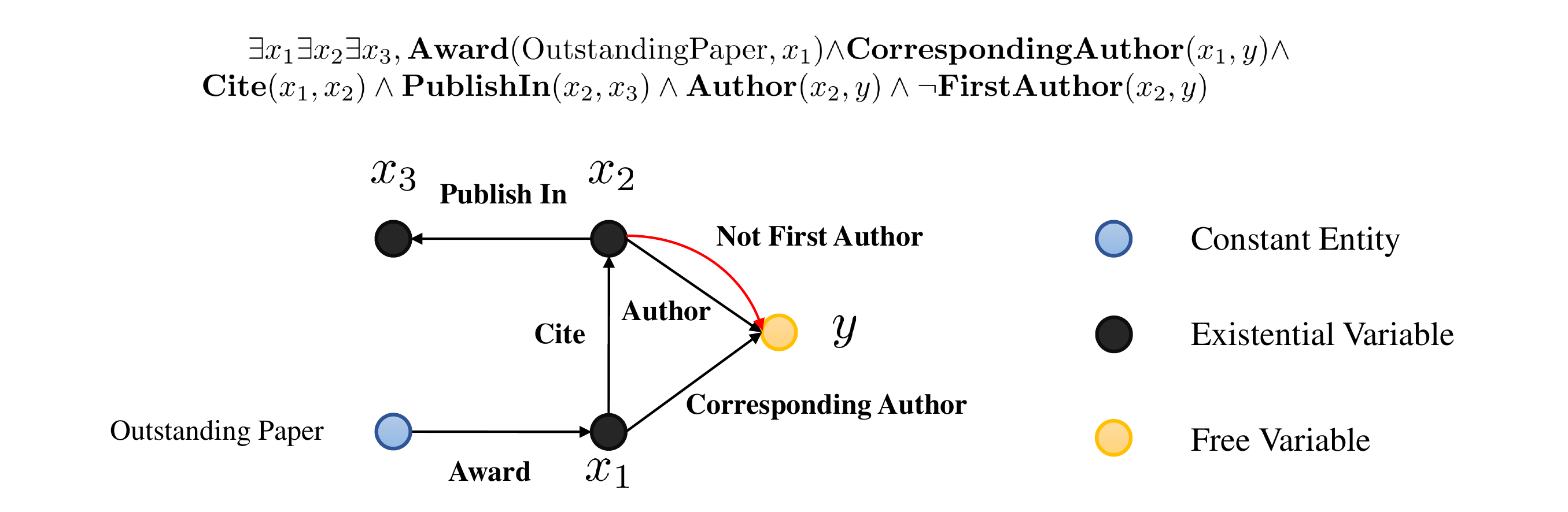}
    \caption{A given query ``Find someone who has such two papers: s/he is the corresponding author of the first paper which has been awarded Outstanding Paper, s/he is the author but not the first author of the second paper which has been published and been cited by the first paper.'' can be represented as such query graphs. The formal existential formula is also given at the top. }\label{fig:query graph}
\end{figure*}

\begin{property}[Negation without Constants]\label{ppt:neg}
There is a negated atomic formula of the form $\lnot r(x_1,x_2)$ or $\lnot r(x_1,y)$ where $x_{1},x_{2}$ are existential variables and $y$ is the free variable. 
\end{property}

\begin{property}[Existential Leaves]\label{ppt:leaf}
For any spanning tree by using the free variable as the root, the leaf node of the query graph can be an existential variable rather than a constant entity.
\end{property}

% \begin{property}[Multi-graph]\label{ppt:multi}
% Multi graph contains self loop and multi edges
% We note that multiple projections $\{u|\exists v\in P, r_1(v,u)\land r_2(v,u)\}$ can not be deduced into simple projection $\{u|\exists v\in P, r(v,u)\}$
% \end{property}
%  example here:  find someone who has published a paper in ICML and both the corresponding author and the first author is himself.

% \begin{property}[Cyclicity]\label{ppt:circle}
% If we allow the computation graph by cyclic, this will break the topology order guaranteed by the tree form. In other words, we will have no idea which node embedding to update next.
% \end{property}

% \begin{example}
%     Discuss Figure~\ref{fig:query graph}
% \end{example}

% \subsection{Categorizing $\efo$ Queries}

\begin{assumption}[Reverse relation enrichment]\label{ass: reverse relation}
    For any relation $r$, there exists a reversed $r^{\prime}$ such that $r^{\prime}(a,b) = r(b,a)$ for any pair of entities $a,b$.
\end{assumption}

\begin{lemma}\label{lem: subformula sentence}
    A query that has a sentence as its subformula is trivial and should not be considered. 
\end{lemma}

We note Assumption~\ref{ass: reverse relation} is a common practice in previous research~\citep{lacroix_canonical_2018}, as well as the Lemma~\ref{lem: subformula sentence}~\citep{ren_query2box_2020,ren_beta_2020}.

\begin{theorem}\label{thm:reduce to TF}
    With assumption~\ref{ass: reverse relation} and lemma~\ref{lem: subformula sentence}, any $\efo$ query with a simple acyclic query graph that does not have Property~\ref{ppt:neg} and~\ref{ppt:leaf} is a $\tf$ query.
\end{theorem}

 In this way, we fill the syntactic gap by making it clear what kinds of queries are left to be discussed, in order to achieve the ambitious goal of answering $\efo$ queries. The proof of Lemma~\ref{lem: subformula sentence} and Theorem~\ref{thm:reduce to TF} is detailed in Appendix~\ref{app:proof of reduce to TF}.

\subsection{Complexity gap between $\tf$ and $\efo$ queries}\label{sec:challenge}

In this part, we further rethink the complexity of the currently investigated queries, particularly, the traditional claim that ``reasoning involves an exponential growth in computational time"~\citep{ren_beta_2020,chen_fuzzy_2022}.

When discussing the complexity, we assume the knowledge graph $\mathcal{G}$ is complete for simplicity. For general $\efo$ queries, it has long been known to be \textbf{NP-complete}~\citep{chandra_optimal_1977}.

However, $\tf$ queries discussed in previous literature are particularly simple. Queries in $\tf \cap \efo$, are well-known special cases with structure-based \textbf{tractability} in inference~\citep{vardi_constraint_2000}. Specifically, we adapt the complexity results from~\citet{dechter_network-based_1987} on knowledge graphs.
\begin{proposition}[Adapted from~\citet{dechter_network-based_1987}]\label{thm:tf}
    The optimal complexity of answering $\tf$ queries in previous datasets~\citep{ren_beta_2020,wang_benchmarking_2021} is $O(n k^2)$, where $n$ is the number of terms and $k$ is a coefficient to characterize the sparsity of the knowledge graph
    \begin{align}
        k = \max_{r\in \relation} |\{ a\in \entity| \exists b. (a,r,b)\in \mathcal{G} \text{ or } (b, r, a) \in \mathcal{G}\}| 
    \end{align}
\end{proposition}
The more detailed proof is provided in Appendix~\ref{app:tree form complexity} where we show we extend traditional results that can only be applied to $\tf \cap \efo$.
 We found that the complexity grows linearly with the number of variables in the query. The complexity depends on the property of knowledge graphs (see the coefficient $k$) and is at most quadratic. Our theorem provides a strict lower bound for previous optimization methods like QTO~\citep{bai_answering_2023}.

We note that the entire $\efo$ queries are inherently harder than $\tf$ queries. To the best of our knowledge, none of the existing methods have considered $\efo$-$\tf$ queries, and their models lack the ability to represent those queries precisely in syntax.

\begin{figure}[t]\label{fig:FIT toy example}
    \vspace{-2em}
    \centering
    \includegraphics[width=\linewidth]{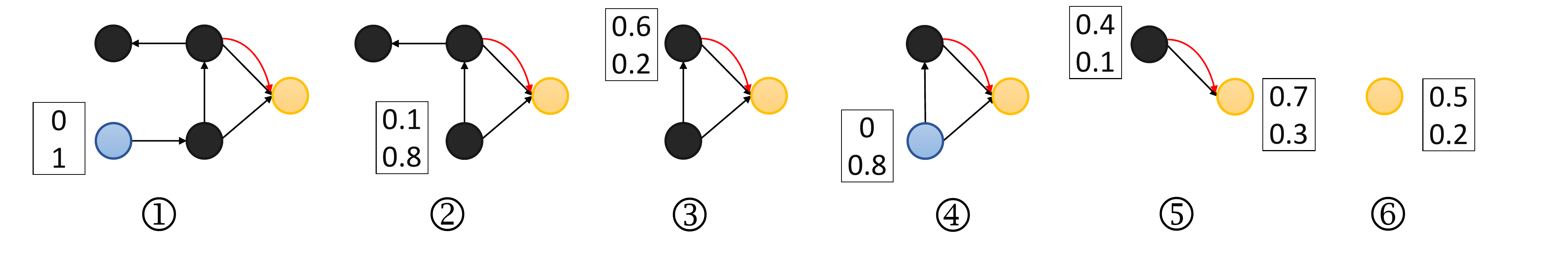}
    \caption{A toy example to show the process of FIT, the vector indicates the fuzzy set of the corresponding node has been updated in this step. The query graph follows Figure~\ref{fig:query graph} with the grounded entity and relation omitted.}
\end{figure}

\section{Methodology}\label{sec:methodology}

This section presents Fuzzy Inference with Truth values (FIT) algorithm for answering $\efo$ queries on knowledge graphs. FIT is a neural-symbolic method that fine-tunes neural link predictors to address the open-world assumption. Specifically, FIT accepts a wide range of models as long as they provide a ``truth value'' for a given triple~\citep{trouillon_complex_2016,cohen_tensorlog_2016,zhu_neural_2021}. Moreover, we prove that our algorithm ensures faithfulness and perfectness with assumptions.

\subsection{Fuzzy truth value definition}

For simplicity, we follow ~\citet{cohen_tensorlog_2016} to conceptualize the neural link predictor as $|\relation|$ matrices $P_r\in [0, 1]^{|\entity|\times|\entity|}$, where $P_r(a, b)$ is the truth value of triple $(a, r, b)$. We note that $r, a,b$ can be treated as integers by indexing all relations and entities. 
% Following~\citet{cohen_tensorlog_2016}, we want to store the truth value of any triples by a large matrix. For any relation $r$, $P_r$ is a $|\entity| * |\entity|$ sparse matrix. The creation of our matrix is explained in Section~\ref{sec:matrix}. We define the truth value function $T$ as the following:

%Then we are ready to define the truth values inductively for any possible substitution $\phi(a/y)$ of an $\efo$ query. With substitution, we only need to consider the formulae without any free variable. 
It is straightforward to use fuzzy logic theory, and define the truth value function $T$ as the following:
\begin{definition}[Truth value of existential formulas]\label{def:tv-formula}
    Let $\phi$ and $\psi$ be existential formulas, $\top$ and $\bot$ are $t$-norms and $t$-conorms, $\bot^{\star}$ is another $t$-conorm, and $r\in \relation$, $a,b\in \entity$.
    \begin{compactenum}
    \item[(i)] $T(r(a, b))=P_r(a, b)$
    \item[(ii)] $T(\lnot \phi) = 1 - T(\phi)$
    \item[(iii)] $T(\phi \land \psi) = T(\phi) \top T(\psi)$, $T(\phi \lor \psi) = T(\phi) \bot T(\psi)$
    \item[(iv)] $T(\exists x \phi(x))=\bot^{\star}_{a\in \entity} T(\phi(a))$
    \end{compactenum}
\end{definition}
For the introduction of the $t$-norm and $t$-conorm, please refer to Appendix~\ref{app:t-norm introduction}.
This definition follows the inductive definition of the existential formula in Definition~\ref{def:formula}, so if the formula is a sentence, it has a certain truth value. 
Moreover, we note that $\bot^*$ is commonly chosen to be Godel $t$-norm~\citep{klir_fuzzy_1995,hajek_metamathematics_2013} but also extended to others in recent research~\citep{van_krieken_analyzing_2022}. 

% Then, we only need to define the truth value of $\exists x \phi(x)$. The standard fuzzy logic theory~\citep{klir_fuzzy_1995, hajek_metamathematics_2013} proposes that $T(\exists x \phi(x))=\max_{a\in \entity} T(\phi(a))$. However, since we are computing on a finite entity set, it has also been proposed that any t-conorm $\bot^{\star}$ may induce an existential quantifier as  by the recent research\citep{thiele_t-quantifiers_1994, van_krieken_analyzing_2022}. We use the $\star$ to denote that this can differ from the truth value of disjunction.
% \begin{definition}[Truth value of existential quantifier]
    
% \end{definition}

% \begin{proposition}
%     The truth value of every existential formula is defined by the above.
% \end{proposition}

% By this definition and recursive computing, we can compute truth value for arbitrary EFO queries. To solve the $\efo$ query, we tackle the free variable by the following definition:

Then, to answer a $\efo$ query is to estimate the truth values of all possible answers as the corresponding substitutions. The more plausible answers are expected to have a higher truth value. We store the truth values in the answer vector as our goal.
\begin{definition}[Answer Vector]
    For $\efo$ query, $\phi(y)$, its answer vector $A[\phi(y)]\in [0, 1]^{|\entity|}$ is given by $A[\phi(y)](a) = T(\phi(a/y))$.
\end{definition}
%Recall that we directly use entity id to represent the entity. The goal of our FIT algorithm is to infer the answer vector of a query $\phi(y)$.

\subsection{Differentiable neural matrix construction}

Neural link predictor is a differentiable model that provides a scoring function $s$ for each possible triple (a,r,b). To fit into our definition and construct the neural matrices, we calibrate the real number score~$s(a,r,b)$ to probability within $[0,1]$ interval by using the softmax function:
\begin{equation*}
    P_{r,a}^{\star}(b) = \frac{exp(s(a,r,b))}{\Sigma_{c\in \entity} exp(s(a,r,c))}
\end{equation*}

As the softmax outputs a vector that has the sum of 1, we further compute the scaling:
\begin{equation}
    Q_{a,b}= 
\begin{cases}
    \frac{|\{d|(a,r,d)\in\mathcal{G}_o\}|}{\Sigma_{c\in \{d|(a,r,d)\in\mathcal{G}_o\}} P_{r,a}^{\star}(c)},& \text{if } |\{d|(a,r,d)\in\mathcal{G}_o\}|> 0\\
    1,              & \text{if } |\{d|(a,r,d)\in\mathcal{G}_o\}|= 0
\end{cases}
\end{equation}

Therefore, the $a$-th row of $r$-th matrix is got by clamping the value for each element:
\begin{equation}
    P_{r}(a,b)= min(1,P_{r,a}^{\star}(b) \times Q_{a,b})
\end{equation}

In testing, the construction of the neural matrix is a bit complicated, details in Appendix~\ref{app:matrix}.

\subsection{Forward inference with truth value}

We explain how to compute the answer vector given neural matrices by a toy example. Full detail of the rigorous derivation is in Appendix~\ref{app:methodology details} and the analysis of its complexity is in Appendix~\ref{app:complexity of FIT}. We infer a fuzzy vector $C_{u}\in[0,1]^{|\entity|}$ to represent the fuzzy set for every node $u$ in the query graph, with this help, we remove node and edge step by step and updating $C_u$ for corresponding nodes simultaneously, keeping the answer vector the same on a smaller graph. Our techniques extend traditional symbolic search methods, like~\citet{yannakakis_algorithms_1981} in database theory and~\citet{dechter_network-based_1987} in constrain satisfaction problem, see more discussion in Appendix~\ref{app: connection to traditional}.

To illustrate our FIT briefly, we explain each step shown in Figure~\ref{fig:FIT toy example}. 1. We initialize the constant entity by giving it a one-hot vector as its fuzzy set. 2. We remove the constant entity and update the fuzzy set of the nodes that are connected to the constant. 3. We remove a leaf variable and update the fuzzy vector of its neighbor. 4. Since it is a circle in the graph, we enumerate possible candidates in the fuzzy set of a variable and make it a constant. 5. Similar to step 2, we remove the constant. 6. The query graph contains only the free variable and the final answer is naturally its fuzzy set. We note this example has included all three features that should be considered according to our Theorem~\ref{thm:reduce to TF}: existential leaf, multigraph, and cyclic graph.

Finally, as all the computations in the FIT are differentiable, the loss is defined as cross-entropy of the predicted answer vector~$A$ with the real answer~$\mathcal{A}$:
\begin{equation}
    \mathcal{L} = H(A,\mathcal{A}) = -[\Sigma_{a\in \mathcal{A}} \ln(A(a)) + \Sigma_{(a\in \entity - \mathcal{A})} \ln(1-A(a))]
\end{equation}

In this way, backward propagation helps to fine-tune the neural link predictors using data of the complex query, rather than just one-hop queries.

\begin{figure}[t]
    \centering
    \includegraphics[width=0.8\linewidth]{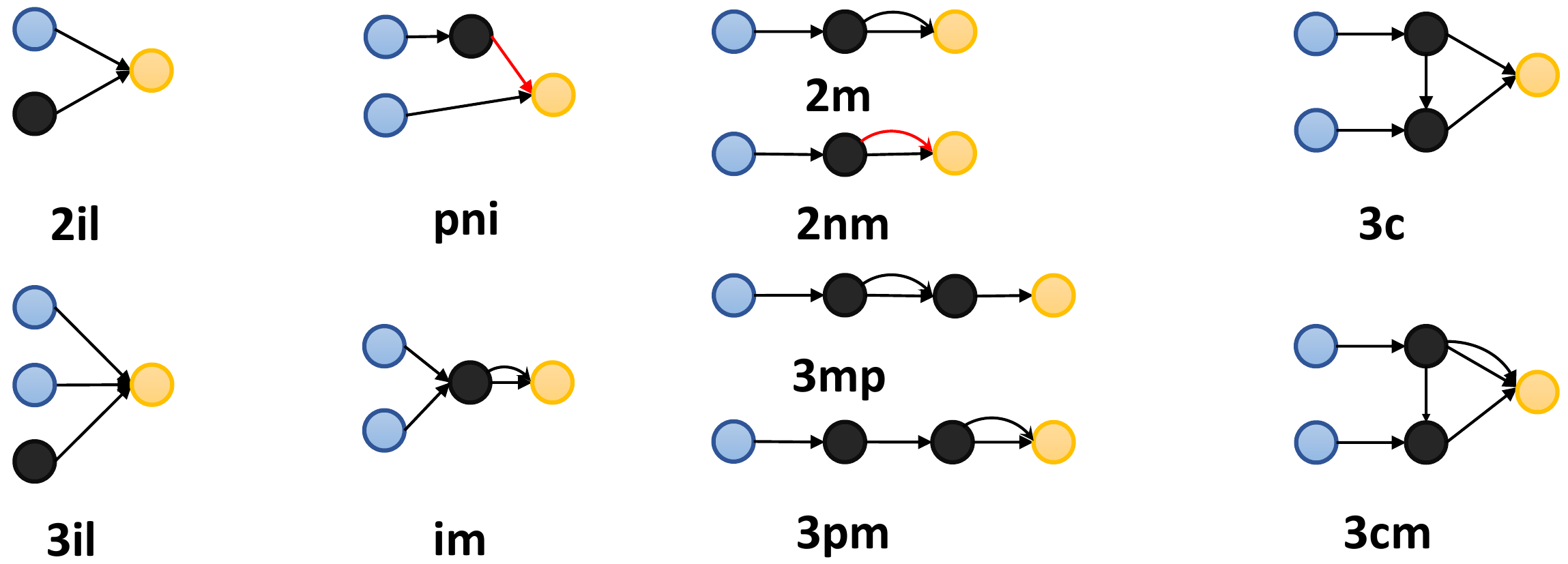}
    \caption{Query graphs of each real $\efo$ formula. Naming convention: l ``existential leaf'', m  ``multi graph'', c for ``circle''. We also follow the previous convention: ``i'' for intersection, ``p'' for projection, and ``n'' for negation. The representation of query graphs follows Figure~\ref{fig:query graph}.}
    \label{fig:new dataset}
    \vspace{-1.5em}
\end{figure}

% \begin{table*}[t]
% \centering
% \scriptsize
% \caption{MRR results(\%) of the tree-form positive queries.}
% \label{tab:EPFO}
% \begin{tabular}{clllllllllll}
% \toprule
% \multicolumn{1}{l}{Knowledge Graph} & Method & 1p & 2p & 3p & 2i & 3i & ip & pi & 2u & up & AVG. \\
% \midrule
% \multirow{3}{*}{FB15k-237} & CQD & \textbf{46.7} & 13.3 & 7.9 & 34.9 & 48.6 & 20.4 & 27.1 & 17.6 & 11.5 & 25.3 \\
%  & LMPNN & 45.9 & 13.1 & 10.3 & 34.8 & \textbf{48.9} & 17.6 & 22.7 & 13.5 & 10.3 & 24.1 \\
%  & FIT & 46.2 & \textbf{14.5} & \textbf{12.9} & \textbf{35.8} & 47.5 & \textbf{21.7} & \textbf{30.2} & \textbf{17.7} & \textbf{13.1} & \textbf{26.6} \\ \midrule
% \multirow{3}{*}{FB15k} & CQD & \textbf{89.2} & 65.3 & 29.7 & 77.4 & 80.6 & 71.6 & 70.6 & 72.3 & \textbf{59.4} & 68.5 \\ 
%  & LMPNN & 85 & 39.3 & 28.6 & 68.2 & 76.5 & 43 & 46.7 & 36.7 & 31.4 & 50.6 \\
%  & FIT & 89.1 & \textbf{65.8} & \textbf{57.2} & \textbf{78.3} & \textbf{81.3} & \textbf{72.0} & \textbf{73.2} & \textbf{74.0} & 59.2 & \textbf{72.2} \\ \midrule
% \multirow{3}{*}{NELL} & CQD & 60.4 & 22.6 & 13.6 & \textbf{43.6} & \textbf{53} & 25.6 & 31.2 & 19.9 & 16.7 & 31.8 \\
%  & LMPNN & 60.6 & 22.1 & 17.5 & 40.1 & 50.3 & 24.9 & 28.4 & 17.2 & 15.7 & 30.8 \\
%  & FIT & \textbf{60.7} & \textbf{23.8} & \textbf{21.2} & 42.4 & 50.7 & \textbf{26.5} & \textbf{31.8} & \textbf{20.3} & \textbf{17.6} & \textbf{32.8} \\ \bottomrule
% \end{tabular}
% \vspace{-1em}
% \end{table*}

\subsection{Theoretical guarantees}

\begin{definition}[Perfect Matrices]\label{def: perfect matrix}
    Given a knowledge graph $\mathcal{G}$, we say the matrices $\{P_r\}_{r\in \relation}$ are perfect if $P_r(h,t) = 1 \iff (h,r,t)\in \mathcal{G}$, $P_r(h,t) = 0 \iff (h,r,t) \notin \mathcal{G}$.
\end{definition}

\begin{theorem}[Perfectness]\label{thm:perfect}
 If the matrices$\{P_r\}_{r\in \relation}$ are perfect, then the FIT algorithm is perfect, meaning: $A[\phi(y)](a)=1 \iff a\in \mathcal{A}[\phi(y)]$, $A[\phi(y)](a)=0 \iff a\notin \mathcal{A}[\phi(y)]$.
\end{theorem}

Though there is no perfect model in practice, we can always assume the given model is ``good'', which admits the following consistency assumption.
\begin{assumption}[Consistent matrix]\label{ass:consistent}
    For any observed triple $(a, r, b) \in \mathcal{G}_o$, $P_r(a, b) = 1$, for any unobserved triple $(a, r, b) \notin \mathcal{G}_o$, $P_r(a, b) < 1$.
\end{assumption}

\begin{definition}
    We say a model is \textbf{faithful} if it is able to retrieve \textbf{deductible} answers like a traditional logical inference system like database query~\citep{kroenke_database_2002,sun_faithful_2020}. 
\end{definition}

\begin{theorem}[Faithfulness]\label{thm:faithful}
    With assumption~\ref{ass:consistent}, for any $\efo$ query $\phi(y)$ without negation, FIT reaches perfect faithfulness, meaning that every answer $a$ that can be deduced in the observed knowledge graph $\mathcal{G}_{o}$, $A[\phi(y)](a)=1$. 
\end{theorem}

The proof of the Theorem~\ref{thm:perfect} and Theorem~\ref{thm:faithful} is provided in Appendix~\ref{app: theoretical guarantee proof}.

\begin{table}[t]
\vspace{-4em}
\centering
\scriptsize
\caption{MRR results(\%) of the new queries on the real $\efo$ dataset.}
\label{tab:new efo1}
\begin{tabular}{cllllllllllll}
\toprule
\multicolumn{1}{l}{Knowledge   Graph} & Method & pni & 2il & 3il & 2m & 2nm & 3mp & 3pm & im & 3c & 3cm & AVG. \\
\midrule
\multirow{7}{*}{FB15k-237} & BetaE & 9.0 & 25.0 & 40.1 & 8.6 & 6.7 & 8.6 & 6.8 & 12.3 & 25.2 & 22.9 & 16.5 \\
 & LogicE & 9.5 & 27.1 & 42.0 & 8.6 & 6.7 & 9.4 & 6.1 & 12.8 & 25.4 & 23.3 & 17.1 \\
 & ConE & 10.8 & 27.6 & 43.9 & 9.6 & 7.0 & 9.3 & 7.3 & 14.0 & 28.2 & 24.9 & 18.3 \\
 & QTO & 12.1 & 28.9 & 47.9 & 8.5 & 10.7 & 11.4 & 6.5 & 17.9 & 38.3 & 35.4 & 21.8 \\
 & CQD & 7.7 & 29.6 & 46.1 & 6.0 & 1.7 & 6.8 & 3.3 & 12.3 & 25.9 & 23.8 & 16.3 \\
 & LMPNN & 10.7 & 28.7 & 42.1 & 9.4 & 4.2 & 9.8 & 7.2 & 15.4 & 25.3 & 22.2 & 17.5 \\
 & FIT & \textbf{14.9} & \textbf{34.2} & \textbf{51.4} & \textbf{9.9} & \textbf{12.7} & \textbf{11.9} & \textbf{7.7} & \textbf{19.6} & \textbf{39.4} & \textbf{37.3} & \textbf{23.9} \\
  \midrule
\multirow{7}{*}{FB15k} & BetaE & 29.9 & 34.8 & 50.6 & 24.4 & 9.6 & 19.0 & 18.4 & 29.1 & 30.5 & 30.7 & 27.7 \\
 & LogicE & 30.7 & 39.3 & 53.0 & 24.1 & 10.5 & 20.5 & 15.5 & 30.7 & 31.8 & 31.7 & 28.8 \\
 & ConE & 37.0 & 40.1 & 57.3 & 33.3 & 11.5 & 23.9 & 27.6 & 38.7 & 35.0 & 36.3 & 34.1 \\
 & QTO & 48.2 & 49.5 & 68.2 & 64.6 & 19.4 & 48.5 & 53.7 & 73.9 & 53.3 & 54.9 & 53.4 \\
 & CQD & 24.2 & 47.6 & 65.4 & 23.2 & 1.6 & 11.0 & 8.7 & 36.3 & 31.3 & 32.9 & 28.2 \\
 & LMPNN & 38.7 & 43.2 & 57.8 & 40.3 & 7.9 & 24.0 & 30.5 & 48.4 & 32.2 & 30.9 & 35.4 \\
 & FIT & \textbf{57.9} & \textbf{70.4} & \textbf{77.6} & \textbf{73.5} & \textbf{39.1} & \textbf{57.3} & \textbf{64.0} & \textbf{79.4} & \textbf{63.8} & \textbf{65.4} & \textbf{64.8} \\
  \midrule
\multirow{7}{*}{NELL} & BetaE & 7.5 & 43.3 & 64.6 & 29.0 & 5.3 & 8.7 & 14.4 & 29.5 & 36.1 & 33.7 & 27.2 \\
 & LogicE & 9.8 & 47.0 & 66.6 & 34.7 & 6.4 & 13.3 & 17.8 & 35.1 & 38.9 & 37.9 & 30.8 \\
 & ConE & 10.3 & 42.1 & 65.8 & 32.4 & 7.0 & 12.6 & 16.8 & 34.4 & 40.2 & 38.2 & 30.0 \\
 & QTO & 12.3 & 48.5 & 68.2 & 38.8 & 12.3 & 22.8 & 19.3 & 41.1 & 45.4 & 43.9 & 35.3 \\
 & CQD & 7.9 & 48.7 & 68.0 & 31.7 & 1.5 & 12.9 & 13.8 & 33.9 & 38.8 & 35.9 & 29.3 \\
 & LMPNN & 11.6 & 43.9 & 62.3 & 35.6 & 6.2 & 15.9 & 19.3 & 38.3 & 39.1 & 34.4 & 30.7 \\
 & FIT & \textbf{14.4} & \textbf{53.3} & \textbf{69.5} & \textbf{42.1} & \textbf{12.5} & \textbf{24.0} & \textbf{22.8} & \textbf{41.5} & \textbf{47.5} & \textbf{45.3} & \textbf{37.3} \\
 \bottomrule
\end{tabular}
 \vspace{-1em}
\end{table}

\begin{figure*}[t]
    \centering
    \includegraphics[width=0.8\linewidth]{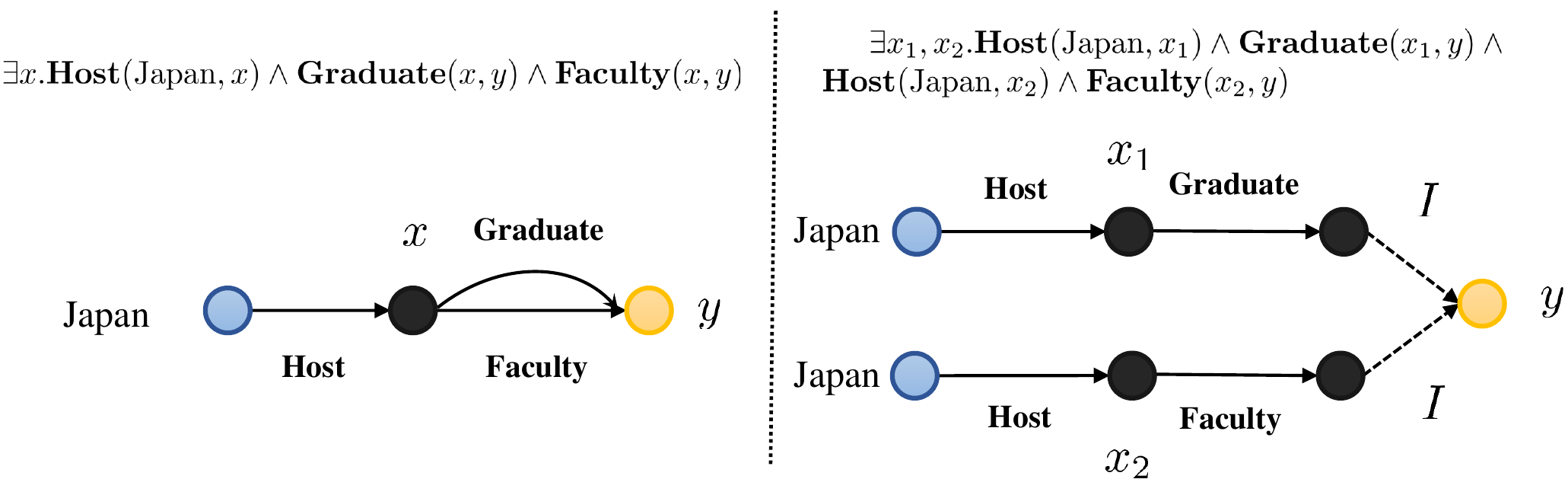}
    \caption{Left: The $\efo$ version. Right: The Tree-Form approximation. The original $x$ is split into $x_1,x_2$ in the Tree-Form query, leading to a different semantic meaning. The formal first order logic formula is also given at the top.}\label{fig:2m example}
    \vspace{-1em}
\end{figure*}

\begin{table}[t]
\vspace{-3em}
\centering
\scriptsize
\caption{MRR results(\%) of the Tree-Form queries. The average score is computed separately among positive and negative queries. The scores of CQD and LMPNN is directly taken from their paper~\citep{minervini_complex_2022,wang_logical_2023}.}
\label{tab:tree form}
\resizebox{\linewidth}{!}{
\begin{tabular}{ccccccccccccccccc}
\toprule
KG & Method & 1p & 2p & 3p & 2i & 3i & ip & pi & 2u & up & AVG.(P) & 2in & 3in & inp & pin & AVG.(N) \\ \midrule
\multirow{3}{*}{FB15k-237} & CQD & \textbf{46.7} & 13.3 & 7.9 & 34.9 & 48.6 & 20.4 & 27.1 & 17.6 & 11.5 & 25.3 &  &  &  &  &  \\
 & LMPNN & 45.9 & 13.1 & 10.3 & 34.8 & 48.9 & 17.6 & 22.7 & 13.5 & 10.3 & 24.1 & 8.7 & 12.9 & 7.7 & 4.6 & 8.5 \\
 & FIT & \textbf{46.7} & \textbf{14.6} & \textbf{12.8} & \textbf{37.5} & \textbf{51.6} & \textbf{21.9} & \textbf{30.1} & \textbf{18.0} & \textbf{13.1} & \textbf{27.4} & \textbf{14.0} & \textbf{20.0} & \textbf{10.2} & \textbf{9.5} & \textbf{13.4} \\ \midrule
\multirow{3}{*}{FB15k} & CQD & 89.2 & 65.3 & 29.7 & 77.4 & 80.6 & 71.6 & 70.6 & 72.3 & \textbf{59.4} & 68.5 &  &  &  &  &  \\
 & LMPNN & 85.0 & 39.3 & 28.6 & 68.2 & 76.5 & 43.0 & 46.7 & 36.7 & 31.4 & 50.6 & 29.1 & 29.4 & 14.9 & 10.2 & 20.9 \\
 & FIT &  \textbf{89.4} & \textbf{65.6} & \textbf{56.9} & \textbf{79.1} & \textbf{83.5} & \textbf{71.8} & \textbf{73.1} & \textbf{73.9} & 59.0 & \textbf{72.5} & \textbf{40.2} & \textbf{38.9} & \textbf{34.8} & \textbf{28.1} & \textbf{35.5} \\ \midrule
\multirow{3}{*}{NELL} & CQD & 60.4 & 22.6 & 13.6 & 43.6 & 53.0 & 25.6 & 31.2 & 19.9 & 16.7 & 31.8 &  &  &  &  &  \\
 & LMPNN & 60.6 & 22.1 & 17.5 & 40.1 & 50.3 & 24.9 & 28.4 & 17.2 & 15.7 & 30.8 & 8.5 & 10.8 & 12.2 & 3.9 & 8.9 \\
 & FIT & \textbf{60.8} & \textbf{23.8} & \textbf{21.2} & \textbf{44.3} & \textbf{54.1} & \textbf{26.6} & \textbf{31.7} & \textbf{20.3} & \textbf{17.6} & \textbf{33.4} & \textbf{12.6} & \textbf{16.4} & \textbf{15.3} & \textbf{8.3} & \textbf{13.2} \\
 \bottomrule
\end{tabular}}
\vspace{-2em}
\end{table}

\begin{table*}[t]
\centering
\scriptsize
\caption{MRR results(\%) of the deductible answers in the BetaE dataset.}
\label{tab:faithful}
\begin{tabular}{rrrrrrrrrrrrrrr}
\toprule
Formula & 1p & 2p & 3p & 2i & 3i & pi & ip & 2in & 3in & inp & pin & pni & 2u & up \\ \midrule
MRR & 100 & 100 & 100 & 100 & 100 & 100 & 100 & 75.5 & 65.3 & 65.2 & 65.7 & 90.5 & 100 & 100 \\ 
 \bottomrule
\end{tabular}
\vspace{-2em}
\end{table*}

\section{Real $\efo$ dataset}\label{sec: new dataset}
Based on our previous discussion of the limitation of query embedding methods in Section~\ref{sec: gap between tree-form and efo}, we have crafted ten real $\efo$ formulas: pni, 2il, 3il, im, 2m, 2nm, 3mp, 3pm, 3c, 3cm. We note these formulas have included all possible new properties of the $\efo$ formulas: Property~\ref{ppt:neg}(pni), Property~\ref{ppt:leaf}(2il,3il), multi-graph(im, 2m, 2nm, 3mp, 3pm, 3cm), and cyclic graph(3c, 3cm). Graphs with self-loop are not included, explained in Appendix~\ref{app:loop}. Additionally, the formula ``pni'' has already been put forward by previous dataset~\citep{ren_beta_2020}, but it is answered as the universal quantifier version as illustrated in Example~\ref{example:pni}, we maintain the query but re-sample its answer according to our definition. The query graphs of our new formulas are presented in Figure~\ref{fig:new dataset}. Its statistics are given in Appendix~\ref{app:statistics}.

\section{Experiment}
\subsection{Settings}

We evaluate our algorithm on various tasks. Firstly, we evaluate our algorithm on our new dataset of real $\efo$ queries developed in Section~\ref{sec: new dataset} and show the failure of existing methods. Secondly, we compare our algorithm with existing methods on the dataset of Tree-Form queries provided by~\citet{ren_beta_2020}.  Thirdly, we verify our claims of Theorem~\ref{thm:faithful} by evaluating also on Tree-Form. 

We follow the standard protocol, splitting the answer into two parts: \textbf{deductible} answers which can be found by the observed knowledge graph $\mathcal{G}_o$, corresponds to Section~\ref{sec: faithful experiment}, \textbf{predicted} answers that need generalization which can be found in the whole knowledge graph $\mathcal{G}$ but not from $\mathcal{G}_o$, correspond to Section~\ref{sec: EFO1 experiment} and \ref{sec: TF experiment}. All our experiments use Mean Reciprocal Rank (MRR) as the metric.

The discussion of the choice of hyperparameters, including the $t$-norm/conorm, is in Appendix~\ref{app:hyperparameter}. Regarding the running speed, please refer to Appendix~\ref{app:complexity of FIT}.

\subsection{Experiment on real $\efo$ dataset}\label{sec: EFO1 experiment}
%The metric of this experiment is only computed among the $\textbf{predicted}$ answers following standard protocol~\citep{ren_query2box_2020}.
 Here we present our result on our new proposed dataset, shown in Table~\ref{tab:new efo1}. As explained in Section~\ref{sec:tf-to-efo}, the Tree-Form query fails to represent the same syntax as our newly developed real $\efo$ queries so it can only $\textbf{syntactically approximate}$ the desired real $\efo$ query. We offer a detailed example of our new ``2m'' query in Figure~\ref{fig:2m example}, where the semantics of the left is ``Find someone who is a graduate of a Japanese University and also a faculty of the $\textbf{same}$ university'' while the right means ``Find someone who is a graduate of a Japanese University and is also a faculty of a Japanese university''. The apparent difference illustrated by this example explains why the previous method falls far behind our FIT, in all types of query and all knowledge graphs. For more details on the implementation of the baseline methods, please refer to Appendix~\ref{app:baseline implement}.

% \exists x_1,x_2. \textbf{Host}(\textrm{Japan},x_1) \land \textbf{Graduate}(x_1,y) \land \textbf{Host}(\textrm{Japan},x_2) \land \textbf{Faculty}(x_2,y), \exists x.\textbf{Host}(\textrm{Japan},x) \land \textbf{Graduate}(x,y)\land \textbf{Faculty}(x,y)

\subsection{Experiments on existing BetaE dataset (Tree-Form)}\label{sec: TF experiment}
As our FIT utilizes a pretrained neural link predictor, for fairness, we use the same checkpoint provided by CQD~\citep{minervini_complex_2022} to compare. %Since CQD has two variants, CQD-beam and CQD-continuous, we adopt the higher one in each query type. 
We note that LMPNN~\citep{wang_logical_2023} also builds its model by the same checkpoint and adds its own new learnable parameters, thus LMPNN is also added as a baseline. Additionally, the connection to QTO is explained in Appendix~\ref{app:connection to QTO}, where we show how FIT is a pure extension to QTO.

Since CQD is not able to answer queries with negation, we split the experiment result into the positive part and the negative part. The result of the positive queries is shown in Table~\ref{tab:tree form}. We have found that our FIT algorithm surpasses both baselines in all knowledge graphs. Most significantly, when the diameter of the query graph gets bigger, like 3p queries, the advantage of our FIT are enormous. As for the result of negative queries, FIT outperforms LMPNN by a large margin in all query types.

\subsection{Experiments on faithfulness}\label{sec: faithful experiment}
To verify our Theorem~\ref{thm:faithful}, we evaluate our algorithm by the $\textbf{deductible}$ answers. The result is shown in Table~\ref{tab:faithful}. To the best of our knowledge, ~\citet{sun_faithful_2020} is the only one that considers faithfulness, we omit its result since it can only infer positive queries in which we have reached perfect faithfulness.

\section{Conclusion}

We reviewed the limitations of current query embedding methods and extended the scope of complex query answering to the whole family of $\efo$ formulas. We also developed a new dataset containing ten formulas and analyzed the new difficulties coming with these formulas. Finally, based on strict fuzzy logic theory, we present a new neural-symbolic method, FIT, which can fine-tune pretrained neural link predictor to infer arbitrary $\efo$ formula and outperform existing methods significantly.

\subsubsection*{Acknowledgments}
The authors of this paper were supported by the NSFC Fund (U20B2053) from the NSFC of China, the RIF (R6020-19 and R6021-20) and the GRF (16211520 and 16205322) from RGC of Hong Kong. We also thank the support from the UGC Research Matching Grants (RMGS20EG01-D, RMGS20CR11, RMGS20CR12, RMGS20EG19, RMGS20EG21, RMGS23CR05, RMGS23EG08). 

\bibliography{ref}
\bibliographystyle{iclr2024_conference}

\appendix

% \section{Related works}

%  \revision{Reasoning on structured knowledge bases, specifically, answering complex queries that are expressed in the formal first order logic on the incomplete knowledge graph has been an important topic in the recent machine learning community. The incomplete nature of the knowledge graph, formally recognized as Open World Assumption, Specifically,}

\section{Missing proofs in the main paper}

We offer all the proof of the results in the main paper in this section.

\subsection{Proof of Lemma~\ref{lem: subformula sentence} and Theorem~\ref{thm:reduce to TF}}\label{app:proof of reduce to TF}
Firstly we proof the lemma~\ref{lem: subformula sentence} by stating it very clearly as the following:

\begin{lemma}
    A query that has a sentence as its subformula is trivial and should not be considered. To be specific, suppose the formula reads $\phi(y_1,\cdots,y_n)$, then it falls into one of the two situations: 1. $\phi(y_1,\cdots,y_n)$ has a certain truth value of 1 or 0, irrelevant of the choice of $y_1,\cdots,y_n$, therefore trivial. 2. There is a subformula $\psi(y_1,\cdots,y_n)$ of the original $\phi(y_1,\cdots,y_n)$ that has the same answer set, meaning the original $\phi(y_1,\cdots,y_n)$ is redundant and should not be considered.
\end{lemma}

\begin{proof}
    We prove that recursively.

    Given a sentence we call $\chi$, w.r.o.t, we assume its truth value is 1.
    
    \begin{itemize}
        \item For another query $\psi$, $\psi \land \chi$ has the same answer set with $\psi$, which falls into the second situation.
        \item For another query $\psi$, $\psi \lor \chi$ always has truth value 1, which falls into the first situation.
        \item $\lnot \chi$ has a certain truth value of 0, which falls into the first situation.
        \item $\exists x \chi$ and $\forall x \chi$ both have a certain truth value of 1 (assume the knowledge graph is not empty), which falls into the first situation. 
    \end{itemize}

    By the recursive definition of first order logic formulas, this lemma is proved.
\end{proof}

We apply this lemma to acyclic queries: 

\begin{lemma}\label{lem:meaningful query}
    With lemma~\ref{lem: subformula sentence}, if a $\efo$ query $\phi(y)$ has an acyclic query graph, then each of its constant entities should be a leaf node in the query graph.
\end{lemma}

\begin{proof}
    We prove by contradiction: if there exists node $a$ which is a non-leaf node but it is also a constant entity. Then we consider the sub-tree whose root is $a$, which represents an existential formula $\psi$ without a free variable (recall that the only free variable $y$ is the root of the whole tree), thus it is a sentence. By lemma~\ref{lem: subformula sentence}, we exclude this kind of query.
\end{proof}

Then we give the proof of Theorem~\ref{thm:reduce to TF}:

\begin{proof}
Considering query $\phi(y)$ whose query graph is acyclic, we compute the spanning tree of it by using the free variable $y$ as the root.  Assumption~\ref{ass: reverse relation} ensures that each edge is from the child node to its parent node. With this representation, we prove the theorem by mathematical induction on the number of nodes in the query graph.

We start when there are two nodes in the query graph as one node can not construct an edge. Then the query must have a form of $r(a,y)$ or $\lnot r(a,y)$ by it does not have property~\ref{ppt:leaf}. Either way, it is a $\tf$ query by definition~\ref{def:TF query}.

Assuming this holds for query graphs with no more than $n\geq 2$ nodes, then we prove it for graphs with $n+1$ nodes.

Find the root $y$, using lemma~\ref{lem:meaningful query}, we know it only connects to existential variables or leaf nodes.

If $y$ connects to a leaf node $u$, we know $u$ must be an entity $a$ by not having property~\ref{ppt:leaf}. Then $\phi(y) = r(a,y)\land \psi(y)$, where $\psi(y)$ is an $\tf$ query by induction. Then using the third part of definition~\ref{def:TF query} finishes the proof.

If $y$ only connects to existential variables, find one child node $u$ with the edge $r$, we note the edge must be positive because of not having property~\ref{ppt:neg}. Then $\phi(y)$ can be decomposed as $\phi(y)=r(u,y)\land \psi(u) \land \chi(y)$, where $\psi(u)$ and $\chi(y)$ are both $\tf$ query by induction (If $y$ only connected to $u$, $\chi(y)$ is omitted). Then using the third and fourth part of definition~\ref{def:TF query} finishes the proof.

Then, the proof is finished by induction.
\end{proof}

\subsection{Proof of the complexity of tree form query}\label{app:tree form complexity}

To begin with, we note that in the construction of the current dataset, there is another assumption called ``bounded negation'' that has been proposed by~\citet{ren_beta_2020} and inherited by~\citet{wang_benchmarking_2021}, which makes sure that set complement is ``bounded'' by set intersection:

\begin{proposition}[Bounded Negation]\label{prop: bounded negation}
    In every operator tree of the previous datasets~\citep{ren_beta_2020,wang_benchmarking_2021}, every node representing ``N'' must have a parent node ``I'', which must have another child node that does not represent ``N''.
\end{proposition}

Then, we follow the algorithm proposed in~\citet{dechter_network-based_1987}, after sorting the query as the operator tree in the DNF, which is just $O(n)$ time, as there are $n$ variables in the operator tree, there are only $n-1$ projection edges in the operator tree. We just need to prove that in computing the answer set for every variable in the operator tree, each projection edge has the complexity of $O(k^2)$. Specifically, we prove a different lemma using the mathematical induction:

\begin{lemma}
    For every variable node $u$ in the operator tree, as we have pointed out in Appendix~\ref{app:tree form derivation}, it represents the answer set of a sub Tree-Form query $\phi(u)$, we prove its answer set $A[\phi(u)]\leq k$.
\end{lemma}

\begin{proof}
    For the leaf variable, since they are all constants, each one just needs $O(1)$ time to start with.

    For a variable $v$, assuming all its child nodes in the operator tree have been computed. By the Proposition~\ref{prop: bounded negation}, we know that there is always a child node $u$ such that there is a positive edge $r$ linked from $u$ to $v$ in the operator tree.
    
    We check the set projection from $u$ to $v$ has the complexity of $O(k^2)$:
    
    Given the answer set $U$ corresponding to node $u$ and satisfying that $|U|\leq k$, given the relation $r$, we prove that the set projection of $U$ satisfies that $|\{v| \exists u\in U. \textbf{ s.t. }  r(c,b)\in \mathcal{G} \}|\leq k$, this is immediately right because of the definition of $k$.

    Then by the nature of the set intersection, the size of the answer set of node $v$ is at most $k$(set union is omitted since the query is in DNF). In this way, we finish the proof by induction.
\end{proof}

Then, we restate the original proposition here and prove it:
\begin{proposition}[Adapted from~\citet{dechter_network-based_1987}]
    The optimal complexity of answering $\tf$ queries in previous datasets~\citep{ren_beta_2020,wang_benchmarking_2021} is $O(n k^2)$, where $n$ is the number of terms and $k$ is a coefficient defined by the knowledge graph
    \begin{align}
        k = \max_{r\in \relation} |\{ a\in \entity| \exists b. (a,r,b)\in \mathcal{G} \text{ or } (b, r, a) \in \mathcal{G}\}| 
    \end{align}
\end{proposition}

\begin{proof}
    By the lemma above, we show that every projection edge is answered in $O(k^2)$ time. Since set complement, intersection, and union only takes $O(k)$ time and the number of these operations is also $O(n)$, we finish the proof that solving the whole Tree-Form query in the Constraint Satisfaction Problem way has the complexity of $O(n k^2)$.
\end{proof}

\subsection{Proof of the theoretical guarantees of FIT}\label{app: theoretical guarantee proof}

In this part, we give proof of the theoretical guarantees of FIT.

\begin{theorem}[Perfectness]
 If the matrices$\{P_r\}_{r\in \relation}$ are perfect, then the FIT algorithm is perfect, meaning: $A[\phi(y)](a)=1 \iff a\in \mathcal{A}[\phi(y)]$, $A[\phi(y)](a)=0 \iff a\notin \mathcal{A}[\phi(y)]$.
\end{theorem}

\begin{proof}
    Our definition of fuzzy truth value coincides with conventional truth value when the truth value of any atomic formula is in $\{0,1\}$~\citep{thiele_t-quantifiers_1994}. Thus, if the given matrices are perfect, our computed truth value coincides with the conventional truth value defined in first order logic.
\end{proof}

\begin{theorem}[Faithfulness]
    With assumption~\ref{ass:consistent}, for any $\efo$ query $\phi(y)$ without negation, FIT reaches perfect faithfulness, meaning that every answer $a$ that can be deduced in the observed knowledge graph $\mathcal{G}_{o}$, $A[\phi(y)](a)=1$. 
\end{theorem}

\begin{proof}
    We create perfect matrices by the observed knowledge graph $\mathcal{G}_{o}$ as in Definition~\ref{def: perfect matrix}, the truth value function induced by it is $T^{\prime}$. By Theorem~\ref{thm:perfect}, answer $a$ is deductible in $\mathcal{G}_{o}$ if and only if $T^{\prime}(a)=1$. We note that $T(r(b,c))\geq T^{\prime}(r(b,c))$ for any  $r\in \relation,b,c \in \entity$, then by the non-decreasing of $t$-norm and $t$-conorm, we know that $T(\phi(a))\geq T^{\prime}(\phi(a))=1$. Thus $A[\phi(y)](a)=T(\phi(a))=1$.
\end{proof}

\section{Tree form query derivation}\label{app:tree form derivation}

In this section, we explain the derivation of the tree form query which we introduce in Definition~\ref{def:TF query}.

 Given a relation $r$, the set projection of set $B$ is $\{c| \exists b\in B. \textbf{ s.t. }  r(b,c)\in \mathcal{G} \}$.

For $\phi(y)=r(a,y)$, $ \mathcal{A}[\phi(y)] = \{b|r(a,b)\in \mathcal{G} \}$, which is projection of a single entity.

For general projection, 

\begin{align*}
    & a \in \mathcal{A}[\exists  y.r(y,y^{\prime})\land  \phi(y)]  \\
    \iff &  \exists b \in \entity. (T(r(b,a)) = 1) \land (T(\phi(b))=1)  \\ 
    \iff &  \exists b \in \entity. r(b,a)\in \mathcal{G} \land b \in \mathcal{A}[\phi(y)] 
\end{align*}

Thus we know that $\mathcal{A}[\exists  y.r(y,y^{\prime})\land  \phi(y)]$ is projection of set $\mathcal{A}[\phi(y)]$ with relation $r$.

The derivation of negation:
\begin{align*}
     &a \in \mathcal{A}[\lnot \phi(y)] \\
     \iff &T(\lnot \phi(a/y)) = 1 \\
     \iff &T(\phi(a/y)) = 0 \\
     \iff &a \notin \mathcal{A}[\phi(y)] \\
     \iff &a \in \overline{\mathcal{A}[\phi(y)]}
\end{align*}
where $\overline{\mathcal{A}[\phi(y)]}$ represents the complement set of $\mathcal{A}[\phi(y)]$. And we know $\mathcal{A}[\lnot \phi(y)]=\overline{\mathcal{A}[\phi(y)]}$.
%Considering a formula $\phi(y) =\exists x \psi(x,y)$, where $\psi(x,y)$ is also an existential formula, let its answer set $A[\phi(y)]=B$. Notice that tree-from queries utilize set complement to tackle negation, we consider $\overline{B}$, namely the complement of set $B$,  $b\in \overline{B}$ means $b \notin B$, which says exactly that $\phi(b/y)$ is false, so we know that $\lnot \exists x \psi(x, b)=\forall x  \lnot\psi(x,b)$  is true.
%That is $\forall x  \lnot \psi(x,y)$ have the answer $b$, thus $A[\forall x  \lnot \psi(x,y)]= \overline{B}$.

The derivation of conjunction is given as follows:
\begin{align*}
    & a \in \mathcal{A}[(\phi \land \psi)(y)]  \\
    \iff &  T((\phi \land \psi)(a)) = 1 \\ 
    \iff &  (T(\phi(a)) = 1) \land (T(\psi(a)) = 1) \\
    \iff & a \in \mathcal{A}[\phi(y)] \land a \in \mathcal{A}[\psi(y)] 
\end{align*}

Thus we know $\mathcal{A}[(\phi \land \psi)(y)] = \mathcal{A}[\phi(y)] \cap  \mathcal{A}[\psi(y)]$. The derivation of disjunction is the same.

\section{$t$-norm introduction}\label{app:t-norm introduction}
\begin{definition}[$t$-norm]
    A $t$-norm $\top$ is a function: [0,1] x [0,1] $\rightarrow$ [0,1] that satisfies the following properties:
    \begin{itemize}[leftmargin =*,]
        \item[(i)] Communitavity: $a\top b = b\top a$
        \item[(ii)] Monotonicity: $(a\top b) \leq  (c\top b)$ if $a\leq c$
        \item[(iii)] Associativity: $(a\top b)\top c= a\top (b\top c)$
        \item[(iv)] Neutrality: $a\top 1 =a$ 
    \end{itemize}
\end{definition}

Then the $t$-conorm $\bot$ is directly defined by $a\bot b = 1 - (1-a)\top(1-b)$, which follows the De Morgan's law.

Finally, we introduce some common $t$-norms which are of interest:

\begin{itemize}
    \item[(i)]  Godel: $a\top_{G} b = \text{min}(a,b)$
    \item[(ii)] Product: $a\top_{P} b = a*b$
    \item[(iii)] Łukasiewicz: $a\top_{LK} b = \text{max}(a+b-1,0)$
\end{itemize}

In the main paper, we mainly focus on the Godel and Product $t$-norm.

\section{FIT methodology details}\label{app:methodology details}
In this section, we give the details of the proof used in Section~\ref{sec:methodology}. We elaborate the Fuzzy Inference with Truth Value (FIT) algorithm by the direct derivation from our previous truth value definition. We discuss how to derive the answer vector of conjunctive queries $\gamma_j(y)$. The answer vector of the whole formula can then be computed by Definition~\ref{def:tv-formula}. The main idea is to cut node and edge step by step and infer on a smaller and smaller query graph.

To begin with, we define a new membership predicate (denoted as $\mu$) between a variable $u$ and a vector $C_u$ representing its fuzzy set, which leads to the truth value of whether entity $a$ belongs to a fuzzy set.
\begin{definition}
    Given an entity $a$ and a vector $C\in [0,1]^{|\entity|}$, the truth value of membership predicate $\mu$ is $T(\mu(a, C)) = C(a)$, namely, the $a$-th coordinate of the vector $C$.
\end{definition}
Equipped with those definitions, we are allowed to build our algorithm on a solid fuzzy logic background~\citep{klir_fuzzy_1995}.

\subsection{Step 1: Initialization}
We initialize the fuzzy vectors for every non-entity node in the query graph $G(\gamma_j)$:
for node $u$ indicates either an existential variable or the free variable, we initialize $C_u=\mathbf{1}$ by an all-one fuzzy vector. The original query is updated as $ \gamma(y) \land \mu(u, C_{u})$.
%(2) If node $u$ indicates an entity $a$, we replace all the occurrences of $a$ with an existential variable $x$ with a corresponding fuzzy vector $C_{x/a}$. We use $x/a$ to represent this replacement operation and $C_{x/a}$ is a one-hot vector with the only 1 in the $a$-th coordinate.The original conjunctive query is rewritten as $\exists x . \gamma(y; x/a) \land \mu(x, C_{x/a})$. 
%(2) If node $u$ indicates an entity $a$, we do nothing.
%(3) If node $u$ is the free variable, we also initialize $C_u = \mathbf{1}$, and the original query is updated as $\gamma(y) \land \mu(x, C_{u})$.

%After initializing fuzzy vectors for all nodes, all entities are replaced by existential variables and free variables. 
Let the derived query be $\phi(y)$, which is $\gamma(y)$ conjuncted with membership predicates from all non-entity nodes.

\subsection{Step 2: Remove self-loops}

For any node $u$ that has self-loop, we discuss two situations, (1) $u$ is the answer node, and (2) $u$ is not the answer node. We note that the self-loop of a constant entity makes no sense.

\noindent\textbf{Case1:Node $u=y$ is the answer node.} 
In this case, we randomly select one of its self-loop edges, if it is positive, representing $r(y,y)$, putting every formula that contains $y$ forehead, the query formula $\phi(y)$ reads

\begin{align*}
    \phi(y)=\mu(y,C_{y}) \land r(y,y) \land \psi(y)
\end{align*}

where $\psi(y)$ is a sub formula.

The answer vector $A[\phi(y)]$ is inferred by evaluating the truth value $T(\phi(a/y))$, then for all $a\in \entity$, 

\begin{align*}
    A[\phi(y)](a)& =  T(\phi(a)) = T(\mu(y,C_{y}) \land r(y,y) \land \psi(y)) \\
    &=  C_{y}(a) \top P_{r}(a,a) \top T(\psi(a))
\end{align*}

 By introducing $\odot^{\top}$ as applying $t$-norm $\top$ element-wise, we have the following vector form:
 
\begin{align*}
    A[\phi(y)]&=  (C_{y} \odot^{\top} \diag(P_{r}))  \odot^{\top} A[(\psi(y))] \\
    &= \mu(y,C_{y}\odot^{\top} \diag(P_{r}))\land \psi(y) 
\end{align*}

By this, we show that the self-loop edge can be removed as long as we update the corresponding fuzzy vector simultaneously.

If the edge is negative, meaning it represents a formula $\lnot r(y,y)$, the derivation is similar:

\begin{align*}
    A[\phi(y)](a)& =  T(\phi(a)) = T( \mu(a,C_{y}) \land \lnot r(a,a)\land \psi(a)) \\
    &=  C_{y}(a) \top (1 - P_{r}(a,a)) \top T(\psi(a)) \\
     A[\phi(y)]& = A[  \mu(y,C_{y}\odot^{\top}(1-\diag(P_{r})))\land \psi(y)]
\end{align*}
 
Thus, $A[\phi(y)]$ can be derived once we infer the answer vector of $A[\psi(y)]$, where $\psi(y)$ is the simplified, inner formula.

\noindent\textbf{Case 2:If $u$ represents an existential variable $x$.} 
Without loss of generality,
we assume there is only $n$ positive self-loops $r_1(x, x),\cdots ,r_n(x,x)$. Then, the formula reads $\phi(y)=\exists x. \mu(x, C_{x}) \land r_{1}(x,x) \land \cdots 
\land r_n(x,x) \land \psi(y; x)$, where $\psi(y;x)$ is an existential formula with the free variable $y$ and existential variable $x$ which does not have sekf-loop anymore. For $a \in \entity$, we have:
\begin{align*}
    &T(\phi(a))\\
    &=\bot^{\star}_{b\in \entity} [T(\mu(b,C_{x})\land r_{1}(b,b) \land \cdots \land r_{n}(b,b) \land \psi(a; b)] \\
    & = \bot^{\star}_{b\in \entity} [C_{x}(b)\top P_{r_{1}}(b,b) \top \cdots \top P_{r_{n}}(b,b)\top T(\psi(a;b))] \\
    & = \bot^{\star}_{b\in \entity} [C^{\prime}_{x}(b)\top T(\psi(a;b))]
    \\
    & = T(\exists x. \mu(x,C^{\prime}_{x}) \land \psi(a/y;x))
\end{align*}
% $T(\phi(a))  = \bot^{\star}_{u=b} [T(\mu(b,C_{u})\land r(b,b) \land \psi(b/u,a)]  = \bot^{\star}_{u=b} [C_{u}(b)\top P_{r}(b,b) \top T(\psi(b,a))]$
where $C^{\prime}_{x} = C_{x} \odot^{\top} \diag(P_{r_{1}}) \odot^{\top} \cdots \odot^{\top}  \diag(P_{r_{n}})$. Therefore, we can remove multiple self-loops similarly.

%by updating the fuzzy vector for $x$ into $C^{\prime}_{x}$.

By the end of this step, all self-loops are removed.
% Similarly define $C^{\prime}_{u} = C_{u} \odot^{\top} \diag(P_{r})$, we have $T(\phi(a/y)) = \bot^{\star}_{u=b} [C^{\prime}_{u}(b) \top T(\psi(b,a))] = T(\exists u \mu(u,C^{\prime}_{u}) \land \psi(u,a)) $

% By this, we are also able to remove the loop edge as long as updating the fuzzy vector $C_{u}$ at the same time.

\subsection{Step 3: Remove constant entity}

If there is a node that represents an entity $a$, we can cut these edges from $a$ to other nodes easily.

Considering all nodes that are connected to $a$, there are two situations:(1)$a$ connects to the free variable $y$, and (2) $a$ connects to an existential variable $x$. 

\noindent\textbf{Case 1: The edge connects $a$  to the free variable $y$} 

As the scenario of negative edge and multi-edge has been discussed, without loss of generality, we can assume there is only one edge from $a$ to $y$ and one edge from $y$ to $a$, which are both positive.  We note that we get rid of the Assumption~\ref{ass: reverse relation} naturally. 
The query formula $\phi(y)$ reads
\begin{align*}
    \phi(y)= \mu(y,C_{y}) \land r_{1}(a,y) \land r_2(y,a) \land \psi(y)
\end{align*}
where $\psi(y)$ is a sub formula.
\begin{align*}
    A[\phi(y)](b)&=T(\mu(b,C_{y}) \land r_1(a,b) \land r_2(b,a)\land\psi(b)) \\
    & = C_{y}(b)\top P_{r_1}(a,b)\top P_{r_2}^{\intercal}(a,b) \top T(\psi(b))\\
    & = C_{y}^{\prime}(b)\top T(\psi(b))  = A[C_{y}^{\prime}(y) \land \psi(y)](b)
\end{align*}
where $C_{y}^{\prime}=C_{y} \odot^{\top} P_{r_1}(a) \odot^{\top} P_{r_2}^{\intercal}(a)$. We also show that the inverse relations can be naturally tackled by the transpose of the predicate matrix.

\noindent\textbf{Case 2:The edge connects $a$ to an existential variable $x$.} 

W.r.o.t, we can assume there is only one edge from $a$ to $x$, the derivation is similar with the one before: 
\begin{align*}
    \phi(y) & = \mu(x,C_{x}) \land r(a,x) \land \psi(y;x) \\
    A[\phi(y)](b)  
    & = \bot^{\star}_{x\in \entity}[T(\mu(x,C_{x}) \land r(a,b)\land \psi(b;c))]\\ 
    & = \bot^{\star}_{c\in \entity}[C_{x}(c)\top P_{r}(a,c)\top T(\psi(b;c))] \\
    & = \bot^{\star}_{c\in \entity}[C_{x}^{\prime}(c)\top T(\psi(b;c))] \\
    & = A[\mu(x,C_{x}^{\prime}\land \psi(y))](b) 
\end{align*}

In this way, all edges connected with $a$ is removed, and then node $a$ is also removed. By the end of this step, all nodes in the query graph represents constants are removed as well as these corresponding edges. 

\subsection{Step 4: Cutting leaf node}
In a query graph, we say a node $u$ is a leaf if it only connects to one other node $v$. If there is a leaf node $u$ in the query graph, We show that we are able to efficiently cut the leaf node and do the inference on the remaining graph, w.r.o.t, we can assume there is only one positive edge from $u$ to $v$.  

There are three cases in this step:

\noindent\textbf{Case 1:$u$ and $v$ represent two existential variables.} 

%\begin{proposition}
%    If $u$ is an existential variable, we can cut $u$ in the induced graph and simultaneously  %update the fuzzy vector $C_{v}$ of node $v$ to keep the answer vector unchanged.
%\end{proposition}

%\begin{proposition}
%     If $u$ is the answer node, we can compute the answer vector by the sub-formula.
%\end{proposition}
%\begin{proposition}
%    $T(\exists v_1 \cdots \exists v_k \phi) = T(\exists v_1 \cdots %\exists v_k \Phi)$ if and only if for any substitution $v_1 = s_1 %\cdots, v_k = s_k$, there is $T(\phi(v_1=s_1, v_k=s_k)) = %T(\Phi((v_1=s_1, v_k=s_k))$.
%\end{proposition}

Similarly, w.l.o.g., we can assume the formula reads:
\begin{align*}
   \phi(y) = \exists x_1,x_2. \mu(x_1, C_{x_1}) \land r(x_1,x_2)  \land \mu(x_2, C_{x_{2}}) \land \psi(x_{2},y) 
\end{align*}

We want to update the candidate vector $C_{x_2}$ instead of really enumerating all possible assignments of $x_1$. That is to say, we want to find $C_{x_2}^{\prime}$ such that for every possible $a\in \entity$, it satisfies:
\begin{equation*}
    \begin{aligned}
        T(\phi(a)) = T(\exists x_2. \mu(x_2, C_{x_2}^{\prime})\land \psi(x_2,a)) 
    \end{aligned}    
\end{equation*}

For simplification, we define an operation $\circ^{\top}$ that for matrix $M,N$ and vector $v$, $M \circ^{\top} v = N \iff N(i,j) = M(i,j)\top V(j)$, let $C^{\star}=  [(P_{r} \circ^{\top} C_{x_2})^{\intercal} \circ^{\top} C_{x_1}]^{\intercal}$, and use the commutativity of $t$-conorm, the original formula can be reformulated as:
\begin{equation}\label{eq:two existential}
    \begin{aligned}
        T(\phi(a)) &= \bot^{\star}_{x_1=b,x_2=c} [C_{x_1}(b) \top P_{r}(b, c) \top C_{x_2}(c) \top T(\psi(c,a))] \\
     & = \bot^{\star}_{x_2=c} \{\bot^{\star}_{x_1=b} [C^{\star}(b,c) \top T(\psi(c,a))]\} 
    \end{aligned}
\end{equation}
And the desired form is 
\begin{equation}\label{eq:one existential}
    T(\exists x_2. \mu(x_2, C_{x_2}^{\prime})\land \psi(x_2,a)) = \bot^{\star}_{x_2=c}[C_{x_2}^{\prime}(c) \top T(\psi(c,a))]
\end{equation}
Compare equation~\ref{eq:two existential} and ~\ref{eq:one existential},  since  $T(\psi(c,a))$ is completely unknown, we want it hold for every $a,c$ that:
\begin{align*}
    \bot^{\star}_{x_1=b} [C^{\star}(b,c) \top T(\psi(c,a))] =  C_{x_2}^{\prime}(c) \top T(\psi(c,a))
\end{align*}

We note this can not be done by arbitrary $t$-conorm. However, if the $\bot^{\star}$ is Godel, namely max, then by the nondecreasing of $t$-conorm we have an important conclusion:
\begin{equation}\label{eq: max exchange}
    \max_{b} [C^{\star}(b,c) \top T(\psi(c,a))] = \max_{b} [C^{\star}(b,c)] \top T(\psi(c,a))
\end{equation}

Then we finally get desired $C_{x_2}^{\prime}(c)=\max_{b} [C^{\star}(b,c)]$.

\noindent\textbf{Case 2:$u$ is the free variable $y$, $v$ is existential variable $x$.} 

Then, the formula $\phi(y)$ reads:
\begin{align*}
    \phi(y)&=\mu(y, C_y)  \land [\exists x. r(x,y)  \land \psi(x)]
\end{align*} where $\psi(x)$ is a formula with $x$ being the only free variable.
\begin{align*}
    A[\phi(y)](a) &  = C_y(a) \top T(\exists x. r(x,a)  \land \psi(x)) \\
    & = C_y(a) \top \{\bot^{\star}_{x=b} [P_{r}(b,a) \top \psi(b)]\}
\end{align*}
If we define
\begin{align*}
    A^{\star} = P_{r}  \circ^{\top} A[\psi(x)]
\end{align*}
where $\circ^{\top}$ is defined samely as in main paper.
We can have the simplification:
\begin{align*}
    A[\phi(y)](a) & = C_y(a) \top [\bot^{\star}_{x=b} A^{\star}(b,a)]
\end{align*}

In this way, we can derive the final answer vector $A[\phi(y)]$ by computing the answer vector $A[\psi(x)]$ first.

\noindent\textbf{Case 3:If $u$ is an existential variable $x$ and $v$ is the free variable $y$.} 
We also assume there is only one edge from $u$ to $v$:

The formula $\phi(y)$ reads:
\begin{align*}
    \phi(y) = \exists x.\mu(x,C_{x}) \land r(x,y) \land \mu(y,C_y) \land \psi(y)
\end{align*}
The derivation is as follows:
\begin{align*}
    A[\phi(y)](a)=T(\phi(a)) & = \bot^{\star}_{x=b} [ C_{x}(b)  \top P_{r}(b, a) \top C_{y}(a)  \top T(\psi(a))] \\
    & = \bot^{\star}_{x=b} [C_{y}^{\star}(b,a) \top T(\psi(a))]
\end{align*}
where $C_{y}^{\star}=  [(P_{r} \circ^{\top} C_{y})^{\intercal} \circ^{\top} C_{x}]^{\intercal}$.

Then we found that this is the same as the situation in equation~\ref{eq: max exchange} which we have explained  already. 

To be specific, provided $\bot^{\star}$ is Godel, we have the following derivation;
\begin{align*}
    A[\phi(y)](a) & = \max_{x=b} [C_{y}^{\star}(b,a) \top T(\psi(a))] \\ 
    & = \max_{x=b}[C_{y}^{\star}(b,a)] \top T(\psi(a))
\end{align*}

Let $C_{y}^{\prime}(a)=\max_{x=b}[C_{y}^{\star}(b,a)]$ do the trick.

 By the end of this step, all leaf nodes are removed.

\subsection{Step 5: Enumeration on circle}

We deal with the circle in the query graph: the only technique is cutting one node $x$ and doing the enumeration: $T(\exists x. \phi(x))=\bot^{\star}_{a\in \entity} T(\phi(a))$.

We choose a node $u$ that cutting it can still keep the remaining graph connected, whose existence is guaranteed by graph theory~\citep{gallier_discrete_2011}.

In practice, we can set a hyperparameter $M$ to limit the maximum number of enumerations, by the assumption~\ref{ass:consistent}, we can distinguish observed nodes from predicted ones. For node $u$, we sort all its candidates and enumerate $|\{a\in \entity|C_{u}(a)=1\}| + M$ the most possible ones.

In this step, we choose a node in the query graph to become a constant entity, which allows for returning back to step 3 again. In this way, the query graph becomes smaller and smaller.

\subsection{Step 6: Getting answer vector}

 As all steps assure the connectivity of the query graph, finally, the query graph will only contain the free variable $y$, with the only remaining formula being $\mu(y, C_y)$, then by definition, its answer vector will be $C_y$.

Additionally, we offer the pseudo-code for our FIT algorithm in Algorithm~\ref{alg: FIT whole} and Algorithm~\ref{alg: FIT conjunctive}. We also note that the design of the FIT algorithm is versatile and can be extended to a more complicated version~\citep{fei_soft_2024}.

\begin{algorithm}[t]
\caption{FIT algorithm on any $\efo$ formulas, where FITC is FIT computed on a query graph, explained in Algorithm\ref{alg: FIT conjunctive}.}
\label{alg: FIT whole}
\SetKwFunction{FITC}{\textsc{FITC}}
\KwIn{$\efo$ query $\phi(y)$, Relation matrices $\{P_r\}_{r \in \relation}$}
Change the EFO1 query to DNF, $\phi_{\rm DNF}(y)=  \gamma_1(y) \lor \dots \lor \gamma_m(y)$\;

\For{$\gamma_i$ in $\phi_{\rm DNF}(y)$}
{
    For the conjunctive query $\gamma_i(y)$, create its query graph $G_{\gamma_{i}}$\;
    Do the initialization, get $\{C_u \}_{u\in G_{\gamma_{i}}}$ \;
    Compute each answer $A_{\gamma_{i}}$ =\FITC($G_{\gamma_{i}}, \{C_u\}_{u\in G_{\gamma_{i}}}$) \;
}
 Aggregate the sub answers and get the final answer  $A[\phi(y)]=\bot_{i}[A_{\gamma_{i}}]$\;
 \KwOut{$A[\phi(y)]$}
\end{algorithm}

\begin{algorithm}[t] 
\caption{FIT on a conjunctive query, which is represented by a query graph. We name it FITC for short.}

\SetKwFunction{FITC}{\textsc{FITC}}
\FITC{$G,\{C_u\}_{u\in G}$}{ 

\If{G contains only one node y\label{alg: FIT conjunctive}}
{
\KwRet{$C_y$}
}

\If{G contains a node $u$ with self loop edges.}
{
    Remove the self loop edges and changes $C_u$ simultaneously as explained in Step 2\;
}
\If{G contains constant entity node.}
{
    Remove the constant entity node as explained in Step 3\;
}

\If{G contains a leaf node $u$ which only connects to node $v$} 
{  \If{$u$ represents free variable, $v$ represents an existential variable.}
    {
        Change $v$ to free variable, removing the node $u$, the new query graph is $G^{\prime}$ \;
        sub answer =  \FITC($G^{\prime}$, $\{C_u\}_{u\in G^{\prime}}$) \;
        final answer is computed by sub answer as Step 4, case 2\;
        \KwRet{final answer}
    }
    \ElseIf{$u$ represents an existential variable, $v$ represents the free variable.}
    {
        Update $C_v$ according to Step 4, case 3, getting $C^{\prime}$ \;
        Remove the node $u$, the query graph is $G^{\prime}$ \;
        \KwRet{\FITC($G^{\prime}$, $\{C^{\prime}_u\}_{u\in G^{\prime}}$)}
    }
    \Else
    {
        Update $C_v$ according to Step 4, case 1, getting $C^{\prime}$ \;
        Remove the node $u$, the query graph is $G^{\prime}$ \;
        \KwRet{\FITC($G^{\prime}$, $\{C^{\prime}_u\}_{u\in G^{\prime}}$)}
    }
}
\Else
{
    Find a node $u$ to enumerate, according to step 5 \;
    Construct the candidate list for node $u$, namely the top N biggest index of $C_u$, where $N=M +\{a\in \entity | C_u(a)=1\}$  \;
    Create a matrix $E\in [0,1]^{N*|\entity|}$ to store enumerate answer \;
    \For{The $i$th candidate of node $u$, $a$}
    {
        Store the original truth value of $C_u(a)$ \;
        Change the $C_u$ to the one-hot vector $C_u^{\prime}=\mathbf{1}_{\{x=a\}}$ \;
        Change the node $u$ to represent a constant entity in query graph, creating $G^{\prime}$ \; 
        Enumerate answer A = \FITC($G^{\prime}$, $C^{\prime}$) \;
        $E[i]=A * C_u(a)$
    }
    \KwRet{$\bot^{\star}_{i}[E(i,j)]$}
}
}
\end{algorithm}

\section{Complexity of FIT algorithm}\label{app:complexity of FIT}
We discuss the complexity of FIT algorithm here. For FIT on a general query graph, it is clearly NP-complete due to our discussion in Section~\ref{sec:challenge}. Specifically, FIT has the complexity of $O(|\entity|^n)$, where n is the number of the variable in the query graph. In the worst case, the query graph has no constants and is a complete graph~\footnote{Complete graph means that every node in the graph connects to every other node in the graph.}, and FIT degrades to enumeration because of circles.

However, as we have also discussed in Section~\ref{sec:challenge}, Tree-Form queries are known to be tractable and we would like to discuss FIT complexity in the $\tf \cap \efo$ queries, the result is given as the following:

\begin{lemma}
    For a query graph with $n$ variables, brutal-force implementation of FIT has a complexity of $O(n\times|\entity|^2)$.
\end{lemma}

\begin{proof}
    We prove this lemma easily by noticing that the complexity in Step 3 and Step 4 is all we need to consider. We note that the bottleneck of the complexity comes from the computation of $C_{y}^{\star}=  [(P_{r} \circ^{\top} C_{y})^{\intercal} \circ^{\top} C_{x}]^{\intercal}$, which is $O(|\entity|^2)$, while the other computations have no larger complexity. Because it takes $O(|\entity|^2)$ to remove one node in the query graph, the total complexity is $O(n\times|\entity|^2)$.
\end{proof}

\begin{proposition}
    For a query graph with $n$ variables, an efficient implementation of FIT has a complexity of $O(n\times t^2)$, where t is a coefficient defined by the sparsity of the matrix.
    \begin{align}
        t = \max_{r\in \relation} |\{ a\in \entity| \exists b. P_r(a,b) > 0 \text{ or } P_r(b,a) > 0 \}| 
    \end{align}
\end{proposition}

\begin{proof}
    The proof is straightforward by the same proof technique used in Appendix~\ref{app:tree form complexity}. By the definition of t, every matrix of relation $r$ can be simplified as an $t*t$ dense matrix, and every $C_u$ in the computation process is guaranteed to have no more than $t$ elements that are larger than 0. Then, the optimal complexity of FIT is $O(n\times t^2)$. 
\end{proof}

Finally, we provide readers with the empirical result of the running speed of our FIT algorithm, the result is shown in Table~\ref{tab: FIT running time}. The experiment is done on a single NVIDIA A100 GPU (40GB).

\begin{table}[t]
\centering
\scriptsize
\caption{The running speed of FIT on the $\tf$ and the $\efo$ queries on FB15k-237, the metric is ms/query, namely how many milliseconds one query needs to infer. The experiment is done on a NVIDIA A100 GPU (40GB).}
\label{tab: FIT running time}
\begin{tabular}{lllllllllllllll}
\toprule
  && & pni  & 2il  & 3il  & 2m   & 2nm  & 3mp  & 3pm  & im   & 3c     & 3cm    & AVG.    \\ 

          & &   & 27.7 & 11.2 & 12.6 & 17.2 & 43.2 & 16.8 & 17.4 & 17.1 & 298    & 340    & 80.12  \\ \midrule

  1p   & 2p   & 3p    & 2i   & 3i & ip   & pi   & 2u   & up   & 2in  & 3in  & inp  & pin  & AVG.        \\
 11.1 & 12.2 & 14.23 & 11.7 & 13 & 13.7 & 13.2 & 11.5 & 13.5 & 14.3 & 14.5 & 32.4 & 20.4 & 15.06 \\ \bottomrule
\end{tabular}

\end{table}

\section{Training and testing details}\label{app:matrix}

We utilize pretrained neural link predictors provided by~\citet{minervini_complex_2022}, who followed previous knowledge graph embeddings method~\citep{trouillon_complex_2016} which gives a score $s(h, a, b)$ for every possible triple$(a,r,b)$, where $a,b\in \entity, r\in \relation$. To convert real number scores to truth value that falls into $[0, 1]$, we use the softmax function:

$$P_{r,a}^{\star}(b) = \frac{exp(s(a,r,b))}{\Sigma_{c\in \entity} exp(s(a,r,c))}$$

However, we notice that this implies that there's only one tail node for $(a,r)$, thus we compute the scaling factor by utilizing the observed edges in the observed graph $\mathcal{G}_o$. We define the set $E_{a,r} = \{b\mid (a,r,b)\in \mathcal{G}_o\}$.

\[
    Q_{a,b}= 
\begin{cases}
    \frac{|E_{a,r}|}{\Sigma_{c\in E_{a,r}} P_{r,a}^{\star}(c)},& \text{if } |E_{a,r}|> 0\\
    1,              & \text{if } |E_{a,r}|= 0
\end{cases}
\]

For training, we just clamp the value for each triple:

\[
    P_{r}(a,b)= min(1, P_{r,a}^{\star}(b) * Q_{a,b})
\]

However, in testing, we compute the final truth value by combining the computed probability with the required property in Assumption~\ref{ass:consistent}:

\begin{equation}\label{eq: test matrix}
        P_{r}(a,b)= 
\begin{cases}
    1, & \text{if } b \in E_{a,r}\\
    0,              & \text{if } P_{r,a}^{\star}(b) * Q_{a,b}< \epsilon \\
    \min(P_{r,a}^{\star}(b) * Q_{a,b}, 1- \delta), & \text{otherwise}
\end{cases}
\end{equation}

We note that $\epsilon, \delta$ are hyper-parameters, $\epsilon$ acts as a threshold which ensures for every $r$, $P_r$ is a sparse matrix and $\delta>0$ is required to meet our Assumption~\ref{ass:consistent}.

Once the $P_r$ is a sparse matrix, our complexity discussion in Appendix~\ref{app:complexity of FIT} applies. However, there are some query types that do not require the matrix to be sparse, namely the 1p, 2i, 3i, 2in, and 3in, since there is no existential variable in them. Therefore, the computation becomes simple as we only need to compute one row in the whole matrix. 

Therefore, when dealing with these five query types, we just use the dense matrix:

\begin{equation}
        \overline{P}_{r}(a,b)= 
\begin{cases}
    1, & \text{if } b \in E_{a,r}\\
    \min(P_{r,a}^{\star}(b) * Q_{a,b}, 1- \delta), & \text{otherwise}
\end{cases}
\end{equation}

For the training, we only use 1p, 2i, 3i, 2in, and 3in query types for efficiency, the reason is explained above. The learning rate is set to 0.0001, the batch size is set to 64, the maximum training step is set to 5,000 steps and we choose the best checkpoints by the scores in the validation set. The choice of the hyperparameter is in Appendix~\ref{app:hyperparameter}.

\section{Connections to previous methods}
In this section, we discussion the connections between FIT and previous methods.

\subsection{Connection to traditional method}\label{app: connection to traditional}

It has been traditionally found that acyclicity is the key to tractability: researchers from two different backgrounds, namely the databases theory~\citep{yannakakis_algorithms_1981} and constraint satisfaction problem~\citep{dechter_network-based_1987} has all found out that conjunctive queries can be evaluated in polynomial time if it is acyclic. However, their results are both in classical settings, where the truth value can either be 0 or 1, while our setting is extremely general, not only do we have a probabilistic value for each triple, but our Definition~\ref{def:tv-formula} is the most general one which allows for two arbitrary t-norms be included in the computation. Moreover, we offer efficient computation that is purely vectorized, which is provided in detail in Appendix~\ref{app:methodology details}. Most significantly, FIT is a neural-symbolic method, the neural link predictor can be fine-tuned by complex query to further address the OWA, while the traditional methods can never \textbf{predict} any answer that is not observed in $\mathcal{G}_o$, which is why machine learning techniques should step in.

\subsection{Connections to QTO}\label{app:connection to QTO}

We note that QTO~\citep{bai_answering_2023} is only a simplified version of FIT, in other words, FIT is a natural extension of QTO in the sense that FIT is capable of solving $\efo$ queries that QTO fails to represent syntactically because QTO relies on the operator-tree method. Moreover, please see the proposition below for the inference of queries within  $\tf \cap \efo$.

\begin{proposition}
    FIT can coincide with QTO when answering queries within the $\tf \cap \efo$ family.
\end{proposition}

\begin{proof}
    Since QTO also converts all formulas in DNF, we are allowed to consider only conjunctive queries. Consider a conjunctive query $$\gamma=\exists x_1,\cdots, x_n. \alpha_1 \land \cdots \land \alpha_m $$ where $\alpha_i$ is and atomic formula or its negation. Then, we note that for the objectives of QTO, it aims to find free variable $y$ such that it maximizes:
    \begin{equation}
        \phi(y) = \max_{x_1,\cdots, x_n} T(\alpha_1) \top \cdots \top T(\alpha_m) = \bot^{\star}_{x_1,\cdots, x_n} [T(\alpha_1) \top \cdots \top T(\alpha_m)]
    \end{equation}

where the $\bot^{\star}$ is Godel t-conorm, and the truth of value is negation is also defined by $T(\lnot \phi) =1 - T(\phi)$ in QTO~\citep{bai_answering_2023}. By the analysis, we show that these optimization objectives naturally coincide with the definition of truth value we proposed systematically in Definition~\ref{def:tv-formula}. Thus, QTO and FIT yield the same result in $\tf \cap \efo$ queries as long as these requirements are met: 1. The same learning-base matrix is given with no further fine-tuning, 2. FIT chooses Godel to be its existential $t$-conorm and product to be its conjunctive $t$-norm, and 3: In queries with no existential variable, FIT also uses a sparse matrix rather than the dense ones.
\end{proof}

We have also done experiments to verify this proposition, we show the result in Table~\ref{tab: QTO in tree-form}, in which we show they coincide with each other. In this way, we show that QTO is at best a special case of FIT, and FIT serves as a superior alternative to QTO in all $\efo$ queries.

\begin{table}[t]
\centering
\scriptsize
\caption{The MRR results(\%) of FIT versus QTO on the $\tf \cap\efo$ queries (BetaE dataset) on FB15k-237.}
\label{tab: QTO in tree-form}
\begin{tabular}{lllllllllllllll}
\toprule
Method & 1p & 2p & 3p & 2i & 3i & ip & pi & 2u & up & 2in & 3in & inp & pin & AVG. \\ \midrule
FIT - QTO & 0 & 0 & 0 & 0 & 0 & 0 & 0 & 0 & 0 & 0 & 0 & 0 & 0 & 0 \\
\bottomrule
\end{tabular}
\end{table}

\section{Hyperparameter impact}\label{app:hyperparameter}

In this section, we discuss several hyperparameters used in the FIT algorithm.

For the threshold $\epsilon$, we use 0.005 for both FB15k-237 and FB15k, 0.0002 for NELL. For the $\delta$, it is 0.001 in all datasets. For $M$, it is 10 in all datasets. As for the tnorm, We choose Godel (maximum) to represent the existential quantifier as discussed in equation~\ref{eq: max exchange}, and we choose the product to represent logic conjunction because of its clear probability interpretation, which also follows the previous research~\citep{arakelyan_complex_2020}. Then we study the impact of each hyper-parameter one by one.

Here we present the impact of the hyperparameter, the result is shown in Table~\ref{tab:hyper parameter}.

As shown, the impact of $\delta$ is vert marginal, because very few predicted answers have comparable scores with the observed answers. As for the $\epsilon$, the smaller, the better, which is kind of obvious because the larger threshold monotonically leads to a sparser matrix that loses more information about the sorting of the predicted tail variables. Moreover, we note when inferring 1p,2i,3i,2in, and 3in queries, we do not use the sparse matrix, therefore $\epsilon$ have no impact on these 5 query types. 

\begin{table*}[t]
\centering
\tiny
\caption{Matrix related hyper parameter impact on the MRR results(\%) on the BetaE dataset on FB15k-237.}
\label{tab:hyper parameter}
\begin{tabular}{lllllllllllllll}
\bottomrule
$\epsilon$, $\delta$ & 1p & 2p & 3p & 2i & 3i & ip & pi & 2in & 3in & inp & pin & 2u & up & AVG. \\ \midrule
0.002, 0.001 & \textbf{46.70} & \textbf{14.65} & \textbf{12.87} & 37.52 & \textbf{51.55} & \textbf{22.19} & \textbf{30.46} & \textbf{13.99} & 19.99 & \textbf{10.18} & 9.54 & \textbf{18.04} & 13.09 & \textbf{23.14} \\
0.005, 0.001 & \textbf{46.70} & 14.61 & 12.80 & 37.52 & \textbf{51.55} & 21.88 & 30.10 & \textbf{13.99} & 19.99 & \textbf{10.18} & \textbf{9.55} & \textbf{18.04} & 13.08 & 23.08  \\
0.005, 0.01 & \textbf{46.70} & 14.62 & 12.80 & \textbf{37.54} & 51.54 & 21.88 & 30.10 & \textbf{13.99} & \textbf{20.01} & 10.17 & 9.53 & \textbf{18.04} & \textbf{13.09} & 23.08 \\
0.01, 0.001 & \textbf{46.70} & 14.49 & 12.74 & 37.52 & \textbf{51.55} & 21.57 & 29.37 & \textbf{13.99} & 19.99 & 10.17 & 9.51 & \textbf{18.04} & 13.07 & 22.98
\\ \bottomrule
\end{tabular}
\end{table*}

Then we also discuss the impact of setting the maximum number of enumerations, the result is shown in Table~\ref{tab:hyper parameter m}, where we can see that bigger $M$ leads to better performance, which is intuitively true because taking more possible candidates for inner variables into consideration is beneficial for getting the final answer.

\begin{table}[t]
\centering
\scriptsize
\caption{Max enumeration impact on the MRR results(\%) on the circle query on FB15k-237.}
\label{tab:hyper parameter m}
\begin{tabular}{llll}
\toprule
M & 3c & 3cm & AVG. \\ \midrule
5 & 39.2 & 37.1 & 38.2 \\
10 & 39.4 & 37.3 & 38.4 \\
20 & 39.5 & 37.5 & 38.5 \\ \bottomrule
\end{tabular}
\end{table}

\begin{table}[t]
\centering
\tiny
\caption{The impact of changing conjunction $t$-norm on the MRR results(\%) on FB15k-237.}
\label{tab:other t-norm}
\begin{tabular}{lcllllllllllllll}
\toprule
Method & \multicolumn{1}{l}{Query Type} & pni & 2il & 3il & 2m & 2nm & 3mp & 3pm & im & 3c & 3cm & AVG. &  &  &  \\
Product & \multirow{2}{*}{EFO1} & \textbf{14.9} & \textbf{34.2} & \textbf{51.4} & \textbf{9.9} & \textbf{12.7} & \textbf{19.6} & \textbf{11.9} & \textbf{7.7} & \textbf{39.4} & \textbf{37.3} & \textbf{23.9}  &  &  &  \\
Godel &  &  14.2 & 33.6 & 49.3 & 8.4 & 11.8 & 14.9 & 10.7 & 6.0 & 35.5 & 33.3 & 21.8 &  &  &  \\ \midrule
Method & \multicolumn{1}{l}{Query Type} & 1p & 2p & 3p & 2i & 3i & ip & pi & 2in & 3in & inp & pin & 2u & up & AVG. \\
Product & \multirow{2}{*}{Tree-Form} & \textbf{46.7} & \textbf{14.6} & \textbf{12.8} & \textbf{37.5} & \textbf{51.6} & \textbf{21.9} & \textbf{30.1} & 14.0 & 20.0 & \textbf{10.2} & \textbf{9.5} & \textbf{18.0} & \textbf{13.1} & \textbf{23.1} \\
Godel &  & \textbf{46.7} & 14.1 & 12.0 & 35.1 & 44.7 & 19.7 & 28.3 & \textbf{14.1} & \textbf{20.3} & 10.0 & \textbf{9.5} & 17.9 & 12.5 & 21.9 \\
\bottomrule
\end{tabular}
\end{table}

Finally, we show that it's also possible to choose other $t$-norms for the conjunction, using Godel as the conjunction $t$-norm also has its own advantage in some queries though it is slightly worse than the product on average, shown in Table~\ref{tab:other t-norm}. However, as our computation in equation~\ref{eq:two existential} shows, using Godel as the existential $t$-norm is the only practical way for allowing FIT to compute efficiently.

\begin{table*}[t]
\centering
\scriptsize
\caption{Statistics of our new proposed real $\efo$ dataset.}
\label{tab:statistics}
\begin{tabular}{lllllllllll}
\bottomrule
Knowledge   Graph & pni & 2il & 3il &  2m & 2nm & 3mp & 3pm & im & 3c & 3cm  \\ \midrule
FB15k-237 & $5\times 10^3$ & $5\times 10^3$ & $5\times 10^3$ & $5\times 10^3$ & $5\times 10^3$ & $5\times 10^3$ & $5\times 10^3$ & $5\times 10^3$ & $5\times 10^3$ & $5\times 10^3$  \\
FB15k & $8\times 10^3$  & $5\times 10^3$ & $5\times 10^3$ & $5\times 10^3$ & $5\times 10^3$ & $5\times 10^3$ & $5\times 10^3$ & $5\times 10^3$ & $5\times 10^3$ & $5\times 10^3$  \\
NELL & $4\times 10^3$ & $5\times 10^3$ & $5\times 10^3$ & $5\times 10^3$ & $5\times 10^3$ & $5\times 10^3$ & $5\times 10^3$ & $5\times 10^3$ & $5\times 10^3$ & $5\times 10^3$  \\ 
\bottomrule
\end{tabular}
\end{table*}
%%%%%%%%%%%%%%%%%%%%%%%%%%%%%%%%%%%%%%%%%%%%%%%%%%%%%%%%%%%%%%%%%%%%%%%%%%%%%%%
%%%%%%%%%%%%%%%%%%%%%%%%%%%%%%%%%%%%%%%%%%%%%%%%%%%%%%%%%%%%%%%%%%%%%%%%%%%%%%%

\section{Self loop}\label{app:loop}
Though this circumstance is allowed in logic and explained in our Section~\ref{sec:methodology}, we find that it is rare in real-world KG like NELL.  In fact, there are only four relations in NELL that have self-loop. In FB15k and FB15k-237, most of the self-loop triples are redundant like ``New York is located at New York''. Thus, we do not sample formulas with self-loop, which is also inherited by other work~\citep{yin_textefo_k-cqa_2023,yin_meta_2024}.

\section{Dataset statistics}\label{app:statistics}

The number of queries in the newly developed real $\efo$ dataset is shown in Table~\ref{tab:statistics}. We note that the ``pni'' data retains the previous queries sampled by~\citet{ren_beta_2020}, thus, we retain the number of queries to be different in different datasets. For those brand new queries proposed by ourselves, we sampled 5000 for each query type in each dataset.
%The circle formula in NELL is less than 5000 since NELL is sparser. 

For the formal first order  formulas of our newly developed real $\efo$ dataset, we also offer them in the following Table~\ref{tab:formal efo1 formula of dataset}.

\begin{table}[t]
\centering
\scriptsize
\caption{Formal definition of the new proposed dataset, where $a_i\in \entity$ are entities from knowledge graph, $x_i$ are bounded existential variables, and $y$ is the single free variable.}
\label{tab:formal efo1 formula of dataset}
\begin{tabular}{ll}
\toprule
Name & $\efo$ Formula                                                                           \\ \midrule
pni  & $\exists x. r_1(a_1,x)\land \lnot r_2(x,y)\land r_3(a_2,y)
 $                                                   \\
2il  & $\exists x. r_1(a,y) \land r_2(x,y)$                                  \\
3il  & $\exists x. r_1(a_1,y)\land r_2(a_2,y)\land r_3(x,y)$                                              \\
2m   & $\exists x.r_1(a,x)) \land r_2(x,y)\land r_3(x,y)$                                            \\
2nm  & $\exists x. r_1(a,x)\land r_2(x,y)\land \lnot r_3(x,y)$      \\
3mp  & $\exists x_1,x_2. r_1(a,x_1)\land r_2(x_1,x_2)\land r_3(x_2,y) \land r_4(x_1,x_2)$                             \\
3pm  & $\exists x_1,x_2. r_1(a,x_1)\land r_2(x_1,x_2) \land r_3(x_2,y)\land r_4(x_2,y)$                              \\
im   & $\exists x. r_1(a_1,x)\land r_2(a_2,x) \land r_3(x,y) \land r_4(x,y)$                              \\
3c   &  $\exists x_1,x_2. r_1(a_1,x_1) \land r_2(x_1,y) \land r_3(a_2,x_2) \land r_4(x_2,y) \land r_5(x_1,x_2)$               \\
3cm  & $\exists x_1,x_2. r_1(a_1,x_1) \land r_2(x_1,y) \land r_3(a_2,x_2) \land r_4(x_2,y) \land r_5(x_1,x_2) \land r_6(x_1,y)$ \\
\bottomrule
\end{tabular}
\end{table}

\section{Baseline implementation}\label{app:baseline implement}

In this section, we illustrate the implementation of the baseline method, as we have explained in Section~\ref{sec:tf-to-efo}, previous methods can not directly represent our newly developed dataset. One example in Figure~\ref{fig:2m example} illustrates how the Tree-Form queries can \textbf{approximate} our newly developed queries to their best.

For those methods that depend on the operator tree, which include BetaE~\citep{ren_beta_2020}, LogicE~\citep{luus_logic_2021},  ConE~\citep{zhang_cone_2021}, and QTO~\citep{bai_answering_2023}. similarly with the example in Figure~\ref{fig:2m example}, we manually create a new operator tree to approximate each of our new formulas, which is explained in Table~\ref{tab:lisp language} and Table~\ref{tab:first order formula for tree form approximation}. Because of the limitation of the operator tree form, these formulas are only approximations, we also offer the real logical formula it represents, making the difference with the real $\efo$ formulas in Table~\ref{tab:formal efo1 formula of dataset} more apparent. For the method that utilizes query graph~\citep{minervini_complex_2022,wang_logical_2023}, we use our query graph with their provided checkpoint directly.

\begin{table}[t]
\centering
\scriptsize
\caption{The lisp-like formula for the approximation of our newly developed dataset, it is invented by~\citet{wang_benchmarking_2021} to express any kind of Tree-Form query.}
\label{tab:lisp language}
\begin{tabular}{ll}
\toprule
Name & Lisp-like formula                                    \\
\midrule
pni  & (i,(n,(p,(p,(e)))),(p,(e)))                              \\
2il  & (p,(e))                                             \\
3il  & (i,(p,(e)),(p,(e)))                                 \\
2m   & (i,(p,(p,(e))),(p,(p,(e))))                          \\
2nm  & (i,(n,(p,(p,(e)))),(p,(p,(e))))                        \\
3mp  & (p,(i,(p,(p,(e))),(p,(p,(e)))))                         \\
3pm  & (i,(p,(p,(p,(e)))),(p,(p,(p,(e)))))                 \\
im   & (i,(p,(i,(p,(e)),(p,(e)))),(p,(i,(p,(e)),(p,(e)))))      \\
3c   & (i,(p,(i,(p,(e)),(p,(p,(e))))),(p,(p,(e))))                 \\
3cm  & (i,(i,(p,(p,(e))),(p,(p,(e)))),(p,(i,(p,(e)),(p,(p,(e))))))  \\
\bottomrule
\end{tabular}
\end{table}

\begin{table}[t]
\centering
\scriptsize
\caption{The formal first order logic formula for the approximation of our newly developed dataset.}
\label{tab:first order formula for tree form approximation}
\begin{tabular}{ll}
\toprule
Name  &  Tree Form formula                                        \\
\midrule
pni  &    $\forall x. r_2(b, y)\land (\neg r_1(x, y) \lor \neg r_3(a, x))$                            \\
2il  &        $r(a,y) $                                \\
3il  &        $r_1(a,y)\land r_2(b,y)$                              \\
2m   &       $\exists x_1,x_2. r_1(a,x_1) \land r_2(x_1,y) \land r_1(a,x_2) \land r_3(x_2, y)$                \\
2nm  &   $\exists x_1, \forall x_2. r_1(a,x_1) \land r_2(x_1,y) \land (\neg r_1(a,x_2) \lor \neg r_3(x_2, y))$                  \\
3mp  &    $\exists x_1, x_2, x_3. r_1(a,x_1)\land r_2(x_1,x_2) \land r_1(a,x_3) \land r_4(x_3, x_2)) \land r_3(x_2,y)$                 \\
3pm  &       $\exists x_1, x_2, x_3, x_4. r_1(a,x_1)\land r_2(x_1,x_2) \land r_3(x_2,y) \land r_1(a,x_3) \land r_2(x_3,x_4) \land r_4(x_4,y)$              \\
im   &    $\exists x_1, x_2. r_1(a_1,x_1)\land r_2(a_2,x_1)\land r_3(x_1,y)\land r_1(a_1,x_2)\land r_2(a_2,x_2)\land r_4(x_2,y)$     \\
3c   &  $\exists x_1,x_2,x_3. [r_1(a_1,x_1)\land r_2(x_1,y)] \land [ r_3(a_2,x_2) \land r_1(a_1,x_3)\land r_5(x_3,x_2) \land r4(x_2,y)] $         \\
3cm  & $\exists x_1,x_2,x_3,x_4. [r_1(a_1,x_1)\land r_2(x_1,y) \land r_1(a_1,x_4) \land r_6(x_4,y)]$  $  \land r_3(a_2,x_2) \land r_1(a_1,x_3)\land r_5(x_3,x_2) \land r4(x_2,y) $ \\
\bottomrule
\end{tabular}
\end{table}

\end{document}